\documentclass[10pt,journal,compsoc]{IEEEtran}

\ifCLASSOPTIONcompsoc
  \usepackage[nocompress]{cite}
\else
  \usepackage{cite}
\fi
             
\usepackage[cmex10]{amsmath}
\usepackage{array}
\usepackage{mdwmath}
\usepackage{mdwtab}
\usepackage{eqparbox}
\usepackage{url}
\usepackage{wrapfig}
\usepackage{amssymb}

\usepackage{microtype}
\usepackage{graphicx}
\usepackage{booktabs} % for professional tables
\usepackage{stfloats}
\usepackage{algorithm}
\usepackage{algorithmic}
\usepackage[algo2e,ruled,vlined,boxed,linesnumbered]{algorithm2e}
\usepackage{threeparttable,multirow}
\usepackage{annotate-equations}
\usepackage[skins,listings]{tcolorbox}

\usepackage{xcolor}

\usepackage{amsthm}

\newtheorem{theorem}{Theorem}

\newtheorem{corollary}{Corollary}
\newtheorem{proposition}{Proposition}

\definecolor{color1}{RGB}{129,15,124}
\definecolor{color2}{RGB}{8,81,156}
\definecolor{color3}{RGB}{37,37,37}

\begin{document}

%\title{Simplifying Graph Transformers (SGFormer): \\ Multi-Layer v.s. One-Layer Propagation}
\title{SGFormer: Single-Layer Graph Transformers with Approximation-Free Linear Complexity}
%\title{Single-layer Graph Transformers (SGFormer): Achieving Approximation-free Linear Complexity with Single-layer Propagation}
\author{Qitian Wu, Kai Yang, Hengrui Zhang, David Wipf, Junchi Yan
      % <-this % stops a space
\IEEEcompsocitemizethanks{\IEEEcompsocthanksitem Qitian Wu is with Broad Institute of MIT and Harvard. E-mail: wuqitian@broadinstitute.org. 
\IEEEcompsocthanksitem Kai Yang is with Department of Computer Science and Engineering, Shanghai Jiao Tong University. E-mail: icarus1411@sjtu.edu.cn. 
\IEEEcompsocthanksitem Hengrui Zhang is with University of Illinois, Chicago. E-mail: hzhan55@uic.edu. 
% Chenxiao Yang is with Toyota Technological Institute at Chicago (e-mail: chr26195@sjtu.edu.cn). 
\IEEEcompsocthanksitem David Wipf is with Amazon Web Service. E-mail: davidwipf@gmail.com.
\IEEEcompsocthanksitem Junchi Yan is with School of Artificial Intelligence, Shanghai Jiao Tong University. E-mail: yanjunchi@sjtu.edu.cn.
}% <-this % stops a space
\thanks{This paper is an expanded version from the conference paper \emph{SGFormer: Simplifying and Empowering Transformers for Large-Graph Representations} which is published in Advances in Neural Information Processing Systems (NeurIPS) held on Dec 10-16, 2023 in New Orleans, U.S.}}

%   \author{Qitian Wu,~\IEEEmembership{IEEE Member,}
%       Kai Yang,~\IEEEmembership{IEEE Student Member,}, Hengrui Zhang,~\IEEEmembership{IEEE Student Member,} \\ 
%       David Wipf,~\IEEEmembership{IEEE Fellow}% <-this % stops a space

%   \thanks{Manuscript submitted on July xxx, 2024. \update{This paper is an expanded version from the conference paper \emph{SGFormer: Simplifying and Empowering Transformers for Large-Graph Representations} which is published in Advances in Neural Information Processing Systems (NeurIPS) held on Dec 10-16, 2023 in New Orleans, U.S.}}
%   \thanks{Qitian Wu is with the Broad Institute of MIT and Harvard (e-mail: qitianwu228@gmail.com). Kai Yang is with Shanghai Jiao Tong University (e-mail: icarus1411@sjtu.edu.cn). Hengrui Zhang is with Univeristy of Illinois, Chicago (e-mail: hzhan55@uic.edu). Chenxiao Yang is with Toyota Technological Institute at Chicago (e-mail: chr26195@sjtu.edu.cn). David Wipf is with Amazon Web Service (e-mail: davidwipf@gmail.com).}%  
% }

% The paper headers
\markboth{}{Wu \MakeLowercase{\textit{et al.}}: Towards Simplifying Transformers on Graphs}

% ====================================================================
\IEEEtitleabstractindextext{%
\begin{abstract}
Learning representations on large graphs is a long-standing challenge due to the inter-dependence nature. Transformers recently have shown promising performance on small graphs thanks to its global attention for capturing all-pair interactions beyond observed structures. Existing approaches tend to inherit the spirit of Transformers in language and vision tasks, and embrace complicated architectures by stacking deep attention-based propagation layers. In this paper, we attempt to evaluate the necessity of adopting multi-layer attentions in Transformers on graphs, which considerably restricts the efficiency. Specifically, we analyze a generic hybrid propagation layer, comprised of all-pair attention and graph-based propagation, and show that multi-layer propagation can be reduced to one-layer propagation, with the same capability for representation learning. It suggests a new technical path for building powerful and efficient Transformers on graphs, particularly through simplifying model architectures without sacrificing expressiveness. As exemplified by this work, we propose a Simplified Single-layer Graph Transformers (SGFormer), whose main component is a single-layer global attention that scales linearly w.r.t. graph sizes and requires none of any approximation for accommodating all-pair interactions. Empirically, SGFormer successfully scales to the web-scale graph \textsc{ogbn-papers100M}, yielding orders-of-magnitude inference acceleration over peer Transformers on medium-sized graphs, and demonstrates competitiveness with limited labeled data.
\end{abstract}

% Note that keywords are not normally used for peerreview papers.
}

\maketitle

% For peer review papers, you can put extra information on the cover
% page as needed:
% \ifCLASSOPTIONpeerreview
% \begin{center} \bfseries EDICS Category: 3-BBND \end{center}
% \fi
%
% For peerreview papers, this IEEEtran command inserts a page break and
% creates the second title. It will be ignored for other modes.
\IEEEdisplaynontitleabstractindextext

\IEEEpeerreviewmaketitle

% ====================================================================
% ====================================================================
% ====================================================================

% === I. INTRODUCTION =============================================================
% =================================================================================
\ifCLASSOPTIONcompsoc
\IEEEraisesectionheading{\section{Introduction}\label{sec:introduction}}
\else
\section{Introduction}
\label{sec:introduction}
\fi

\IEEEPARstart{L}{earning} on large graphs that connect interdependent data points is a fundamental challenge in machine learning and pattern recognition, with a broad spectrum of applications ranging from social sciences to natural sciences~\cite{largegraph-app-1,largegraph-app-2,largegraph-app-3,largegraph-app-4,largegraph-app-5}. One key problem is how to obtain effective node representations, i.e., the low-dimensional vectors (a.k.a. embeddings) that encode the semantic and topological features, especially under limited computation budget (e.g., time and space), that can be efficiently utilized for downstream tasks.

Recently, Transformers have emerged as a popular class of foundation encoders for graph-structured data by treating nodes in the graph as input tokens and have shown highly competitive performance on graph-level tasks~\cite{graphtransformer-2020,graphbert-2020,graphtrans-neurips21,graphformer-neurips21,graphgps} and node-level tasks~\cite{gophormer,nodeformer,chen2023nagphormer,wu2023difformer} on graph data. The global attention in Transformers~\cite{transformer} can capture implicit inter-dependencies among nodes that are not embodied by input graph structures, but could potentially make a difference in data generation (e.g., the undetermined structures of proteins that lack known tertiary structures~\cite{nagarajan2013novel,trans-app-1}). This advantage provides Transformers with the desired expressivity for capturing e.g., long-range dependencies and unobserved interactions, and leads to superior performance over graph neural networks (GNNs) in small-graph-based applications~\cite{trans-app-1,trans-app-2,trans-app-3,trans-app-4,trans-app-5}.

However, a concerning trend in current architectures is their tendency to automatically adopt the design philosophy of Transformers used in vision and language tasks~\cite{bert,gpt3,vit}. This involves stacking deep multi-head attention layers, which results in large model sizes and the data-hungry nature of the model. However, this design approach poses a significant challenge for Transformers in scaling to large graphs where the number of nodes can reach up to millions or even billions, particularly due to two-fold obstacles. 

1) The global all-pair attention mechanism is the key component of modern Transformers. Because of the global attention, the time and space complexity of Transformers often scales quadratically with respect to the number of nodes, and the computation graph grows exponentially as the number of layers increases. Thereby, training deep Transformers for large graphs with millions of nodes can be extremely resource-intensive and may require delicate techniques for partitioning the inter-connected nodes into smaller mini-batches in order to mitigate computational overhead~\cite{nodeformer,wu2023difformer,survey-graphtransformer,chen2023nagphormer}.

2) In small-graph-based tasks such as graph-level prediction for molecular property~\cite{hu2021ogb}, where each instance is a graph, and there are typically abundant labeled graph instances, large Transformers may have sufficient supervision for generalization. However, in large-graph-based tasks such as node-level prediction for protein functions~\cite{ogb-nips20}, where there is usually a single graph and each node is an instance, labeled nodes can be relatively limited. This increases the difficulty of Transformers with complicated architectures in learning effective representations in such cases.
    % \item \textbf{Graph Inductive Bias:} Another important question lies in how to incorporate the graph structural information into Transformer architectures. While existing efforts have explored different ways, such as positional embeddings~\cite{graphformer-neurips21,graphtrans-neurips21}, augmented regularization loss~\cite{nodeformer} and combination of GNNs~\cite{graphgps}, existing designs are mostly motivated from subtle heuristics and lack principled guidance. It still remains unclear the underlying mechanism of Transformers and pushing further, if there exist any theoretical guidance for inserting graph inductive bias into Transformers.

This paper presents an attempt to investigate the necessity of using deep propagation layers in Transformers for graph representations and explore a new technical path for simplifying Transformer architectures that can scale to large graphs. Particularly, we start from the interpretation of message-passing-based propagation layers (i.e., a Transformer or GNN layer) as optimization dynamics of a classic graph signal denoising problem. This viewpoint lends us a principled way to reveal the underlying mechanism of graph neural networks and Transformers, based on which we naturally derive a hybrid propagation layer that combines global attention and graph-based propagation in once updates. This hybrid layer can be seen as a generalization of common GNNs' and vanilla Transformer's propagation layers by interpolation between two model classes, and such a design is also adopted by state-of-the-art Transformer models on graphs, e.g.,~\cite{nodeformer,graphgps,wu2023difformer}. On top of this, we answer the following research questions with theoretical analysis, model designs and empirical results as our main contributions.
% \footnote{\textbf{Comparison with the conference version.} Built upon our work~\cite{wu2024simplifying} published in NeurIPS'23, we have made substantial extensions and significantly rewritten the paper. Overall, the conference version \cite{wu2024simplifying} introduces the model from heuristics, while this work starts from a theoretical viewpoint as motivation. 1) We consider a generic hybrid propagation layer with trainable weights and analyze the equivalence between multi-layer and single-layer models (Sec.~\ref{sec-theory}). 2) On top of the theoretical results, we present principled guidance that derives the model architecture of SGFormer with grounded justifications (Sec.~\ref{sec-theory-summary} and \ref{sec-model-design}). 3) Furthermore, we supplement extensive results including comparison with three state-of-the-art competitors ANS-GT~\cite{ANS-GT}, GraphGPS~\cite{graphgps} and DIFFormer~\cite{wu2023difformer} (Sec.~\ref{sec-exp-comp-m}, \ref{sec-exp-comp-l}, \ref{sec-exp-effi}), experiments with limited labeled data (Sec.~\ref{sec-exp-ratio}), ablation studies (Sec.~\ref{sec-exp-ablation}) and thorough comparison between multi-layer and single-layer attentional models (Sec.~\ref{sec-exp-dis}).}. 

% we further prove that the multi-layer hybrid model has the same denoising effect as a one-layer counterpart, which implies a potential means to simplify graph Transformers without sacrificing the expressiveness for learning effective representations. 

\begin{figure*}[t!]
  \centering
  \includegraphics[width=0.95\linewidth]{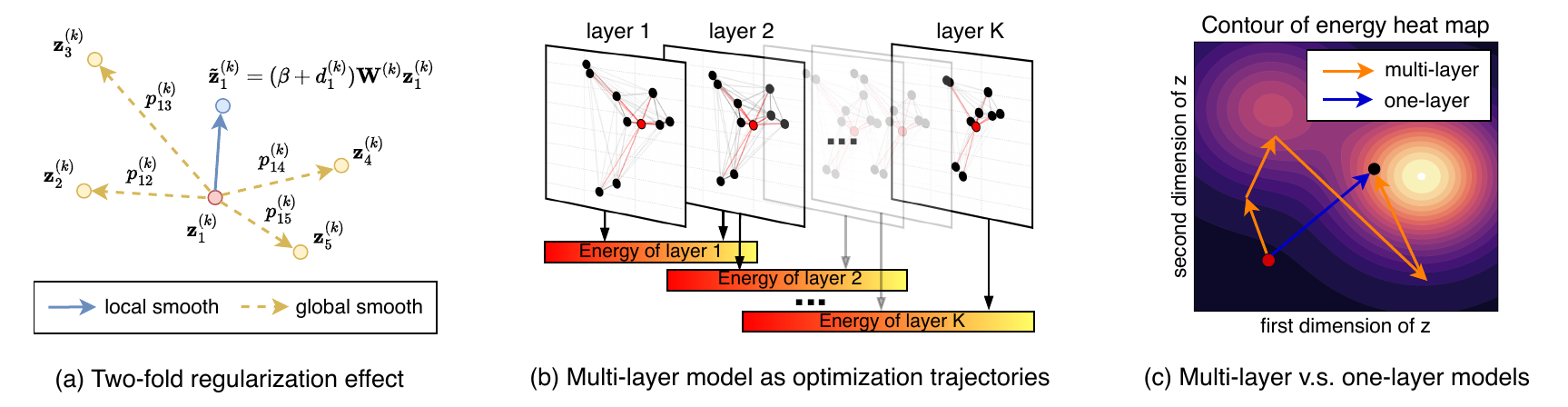}
  % \vspace{-5pt}
  \caption{Illustration of the main theoretical results in Sec.~\ref{sec-theory}. (a) The layer-wise updating rule of message passing models (e.g., GNNs and Transformers) is equivalent to a gradient descent step minimizing a regularized energy in graph signal denoising. The energy has two-fold regularization effects, which enforce local and global smoothness, respectively. (b) Common Transformers stacking multiple propagation layers can be seen as a cascade of descent steps on layer-dependent energy (since the attention scores and feature transformations are specific to each layer). (c) The multi-layer model can be reduced to a one-layer model where the latter contributes to the same denoising effect, i.e., yielding the equivalent output embeddings.}
  \label{fig:theory}
  % \vspace{-5pt}
\end{figure*}

$\blacktriangleright$\textbf{Q1: Is Multi-Layer Propagation Necessary?} We show that one hybrid propagation layer corresponds to a one-step gradient descent on a regularized energy that enforces certain smoothness effects for node representations. In particular, the smootheness effects are two-fold that facilitate local and global regularization (see Fig.~\ref{fig:theory}(a)). Since the energy function depends on the layer-specific attention scores and feature transformation weights, a model involving multi-layer propagation can be seen as a cascade of descent steps minimizing different objectives. In this regard, using multi-layer propagation may lead to potential redundancy from the perspective of graph signal denoising, since the layer-wise updates contribute to disparate smoothing effects and could interfere with each other (see Fig.~\ref{fig:theory}(b)). Mapping back to model architectural designs, such a redundancy challenges the necessity of stacking deep propagation layers in Transformers, which considerably restricts the computational efficiency for scaling to large graphs.

$\blacktriangleright$\textbf{Q2: How Powerful Is Single-Layer Propagation?} Based on the above result, we prove that for any model stacking multiple propagation layers, there exists an energy function such that one-step gradient descent from the initial point yields the equivalent denoising effects as the multi-layer model. Pushing further, there exists a single-layer propagation model whose updated embeddings have negligibly small approximation error compared to the updated embeddings yielded by the multi-layer model (see Fig.~\ref{fig:theory}(c)). Particularly, such a single-layer model also adopts a propagation layer of the hybrid form and corresponds to a single gradient descent step on a fixed energy. The latter, in principle, reduces the redundancy within the optimization dynamics for graph signal denoising. This result implies that multi-layer propagation can be reduced to one-layer propagation that can achieve equivalent (up to negligible approximation error) expressiveness for representations. It further enlightens a potential way to simplify Transformer architectures on graphs without sacrificing effectiveness.

$\blacktriangleright$\textbf{Q3: How to Unleash the Power of Single-Layer Attention?} In light of the theoretical results, we propose an encoder backbone for learning representations on large graphs, referred to as Simplified Graph Transformer (SGFormer). Specifically, SGFormer adopts a hybrid architecture that linearly combines a single-layer global attention and a GNN network. In particular, we propose a simple global attention function that linearly scales w.r.t. the number of nodes and accommodates all-pair interactions without using any approximation. In terms of the GNN network, we simply instantiate it as a shallow GCN whose computation is desirably efficient on large graphs. Equipped with such designs, SGFormer shows expressiveness to capture the implicit dependencies and, meanwhile, incorporates graph inductive bias. Moreover, compared to peer graph Transformers, SGFormer requires no positional encodings, feature pre-processing, extra loss functions, or edge embeddings. 

$\blacktriangleright$\textbf{Q4: How Does A Simple Transformer Model Perform?} Despite using simple architecture, experiments show that SGFormer achieves highly competitive performance in an extensive range of node property prediction datasets, which are used as common benchmarks for model evaluation w.r.t. representation learning on large graphs. In terms of the efficiency, on medium-sized graphs, where the node numbers range from 1K to 10K, SGFormer achieves up to 20x and 30x speedup in terms of training and inference time costs, respectively, over recently proposed scalable Transformers. In terms of scalability, the time and memory costs of SGFormer both scale linearly w.r.t. graph sizes, with lower growth rate than existing linearly-complex Transformers. Notably, SGFormer can scale smoothly to the web-scale graph \textsc{ogbn-papers100M} with 0.1B nodes, where existing models to our knowledge fail to demonstrate. In addition, SGFormer shows superior performance with limited labeled data for training. We also conduct thorough ablation studies that validate the effectiveness of the proposed designs, particularly, the advantage of single-layer attentions over multi-layer ones. The implementation is publicly available at \url{https://github.com/qitianwu/SGFormer}.

\section{Preliminary and Background}\label{sec-related}

In this section, we introduce notations as building blocks of the analysis and the proposed model. In the meanwhile, we briefly review the literature related to the present work. 

\textbf{Notations.} We denote a graph as $\mathcal G = (\mathcal V, \mathcal E)$ where the node set $\mathcal V$ comprises $N$ nodes and the edge set $\mathcal E = \{(u, v)~|~a_{uv}=1\}$ is defined by a symmetric (and usually sparse) adjacency matrix $\mathbf A = [a_{uv}]_{N\times N}$, where $a_{uv}=1$ if node $u$ and $v$ are connected, and $0$ otherwise. Denote by $\mathbf D = \mbox{diag}(\{d_u\}_{u=1}^N)$ the diagonal degree matrix of $\mathbf A$, where $d_u = \sum_{v=1}^N a_{uv}$. Each node has a $D$-dimensional input feature vector $\mathbf x_u \in \mathbb R^D$ and a label $y_u$ which can be a scalar or a vector. The nodes in the graph are only partially labeled, forming a node set denoted as $\mathcal V_{tr} \subset \mathcal V$ (wherein $|\mathcal V_{tr}| \ll N$). Learning representations on graphs aims to produce node embeddings $\mathbf z_u\in \mathbb R^d$ that are useful for downstream tasks. The size of the graph, as measured by the number of nodes $N$, can be arbitrarily large, usually ranging from thousands to billions. 

\subsection{Graph Neural Networks} 

Graph Neural Networks (GNNs)~\cite{scarselli2008gnnearly,GCN-vallina} compute node embeddings through message passing rules over observed structures.
The layer-wise message passing of GNNs can be defined as recursively propagating the embeddings of neighboring nodes to update the node representation:
\begin{equation}\label{eqn-gnnupdating}
    \mathbf z_u^{(k+1)} = \eta^{(k)}(\overline{\mathbf z}_u^{(k+1)}), ~\overline{\mathbf z}_u^{(k+1)} = \mbox{Agg}(\{\mathbf z_v^{(k)}| v\in \mathcal R(u) \}),
\end{equation}
where $\mathbf z_u^{(k)} \in \mathbb R^d$ denotes the embedding at the $k$-th layer, $\eta$ denotes (parametric) feature transformation, and $\mbox{Agg}$ is an aggregation function over the embeddings of nodes in $\mathcal R(u)$. The latter is the receptive field of node $u$ determined by $\mathcal G$. Common GNNs, such as GCN~\cite{GCN-vallina} and GAT~\cite{GAT}, along with their numerous successors, e.g.,~\cite{jknet-icml18,SGC-icml19,h2gcn-neurips20,glognn,zhu2021graph}, typically assume $\mathcal R(u)$ to be the set of neighboring nodes in $\mathcal G$. By stacking multiple layers of local message passing as defined by \eqref{eqn-gnnupdating}, the model can integrate information from the local neighborhood into the representation. 

%Due to the involvement of neighboring nodes in the computation graph, which exponentially increase w.r.t. the layer number, training GNNs on large graphs (e.g., with a million nodes) can be challenging. To reduce the overhead, GNNs require subtle techniques like neighbor sampling~\cite{graphsaint}, graph partition~\cite{clustergcn-kdd19}, or historical embeddings~\cite{gnnautoscale}. 
However, since the number of neighboring nodes involved in the computation exponentially increases as the layer number goes up, the aggregated information from distant nodes will be diluted with an exponential rate w.r.t. the model depth. This problem referred to as over-squashing~\cite{oversquashing-iclr21} can limit the expressiveness of GNNs for learning effective representations.
Moreover, recent evidence suggests that GNNs yield unsatisfactory performance in the case of graphs with heterophily~\cite{h2gcn-neurips20}, long-range dependencies~\cite{longrange-icml18} and structural incompleteness~\cite{IDLS-icml19}. This urges the community to explore new architectures that can overcome the limitations of GNNs' local message passing.

%Some works adopt knowledge distillation for inference on sparser graphs~\cite{yang2022geometric,zhang2022graph} to accelerate inference, and more recently, MLP architectures are used to replace GNNs to accelerate training~\cite{yang2022graph,han2022mlpinit}.

\subsection{Graph Transformers} 

Beyond message passing within local neighborhoods, Transformers have recently gained attention as powerful graph encoders~\cite{graphtransformer-2020,graphbert-2020,graphformer-neurips21,SAN-neurips21,graphtrans-neurips21,SAT-icml22,EGT,nodeformer,graphgps}. These models use global all-pair attention, which aggregates all node embeddings to update the representation of each node: 
\begin{equation}\label{eqn-transupdating}
    \mathbf z_u^{(k+1)} = \eta^{(k)}(\overline{\mathbf z}_u^{(k+1)}), ~\overline{\mathbf z}_u^{(k+1)} = \mbox{Agg}(\{\mathbf z_v^{(k)}| v\in \mathcal V \}).
\end{equation}
The global attention can be seen as a generalization of GNNs' message passing to a densely connected graph where $\mathcal R(u) = \mathcal V$, and equips the model with the ability to capture unobserved interactions and long-range dependence.

However, the all-pair attention incurs $O(N^2)$ complexity and becomes a computation bottleneck that limits most Transformers to handling only small-sized graphs (with up to hundreds of nodes). For larger graphs, recent efforts have resorted to strategies such as sampling a small (relative to $N$) subset of nodes for attention computation~\cite{graphsage} or using ego-graph features as input tokens~\cite{gophormer,chen2023nagphormer}. These strategies sacrifice the expressivity needed to capture all-pair interactions among arbitrary nodes. Another line of recent works designs new attention mechanisms that can efficiently achieve all-pair message passing within linear complexity~\cite{nodeformer,graphgps}. Nevertheless, these schemes require approximation that can lead to training instability. 

Another observation is that nearly all of the Transformers mentioned above tend to stack deep attention layers, in line with the design of large models used in vision and language tasks~\cite{bert,gpt3,vit}. However, this architecture presents challenges for scaling to industry-scale graphs, where $N$ can reach billions. Moreover, due to the complicated architecture, the model can become vulnerable to overfitting when the number of labeled nodes $|\mathcal V_{tr}|$ is much smaller than $N$. This is a common issue in extremely large graphs where node labels are scarce~\cite{ogb-nips20}. The question remains how to build an efficient and scalable Transformer model that maintains the desired expressiveness for learning effective graph representations.

\subsection{Node-Level v.s. Graph-Level Tasks} 

Before going to our methodology, we would like to pinpoint the differences between two graph-based predictive tasks of wide interest. \textit{Node-level tasks} (our focus) target a single graph connecting all the instances as nodes where each instance has a label to predict. Differently, in \textit{graph-level tasks}, each instance (e.g., molecule) itself is a graph with a label, and graph sizes are often small, in contrast with the arbitrarily large graph in node classification depending on the number of instances in a dataset. The different input scales result in that the two problems often need disparate technical considerations~\cite{ogb-nips20}. While GNNs exhibit comparable competitiveness in both tasks, most of current Transformers are tailored for graph classification (on small graphs)~\cite{graphtransformer-2020,graphbert-2020,graphformer-neurips21,SAN-neurips21,graphtrans-neurips21,SAT-icml22,EGT,graphgps,puretransformer}, and it still remains largely under-explored to design powerful and efficient Transformers for node-level tasks on large graphs~\cite{gophormer,ANS-GT,nodeformer}.

\section{Theoretical Analysis and Motivation}\label{sec-theory}

Before we introduce the proposed model, we commence with motivation from the theoretical perspective which sheds some insights on the model design. Based on the discussions in Sec.~\ref{sec-related}, our analysis in this section aims at answering how to simultaneously achieve the two concerning criteria regarding effectiveness and efficiency for building powerful and scalable Transformers on large graphs. We will unfold the analysis in a progressive manner that lends us a principled way to derive the model architecture.

In specific, our starting point is rooted on the interpretation of message-passing-based propagation layers adopted by common GNNs and Transformers as optimization dynamics of a classic graph signal denoising problem~\cite{shuman2013emerging,kalofolias2016learn,pmlr-v162-fu22e}. The latter can be formulated as solving a minimization problem associated with an energy objective and allows a means to dissect the underlying mechanism of different models for representation learning. On top of this, we present a hybrid architecture that integrates the desired expressiveness of all-pair global attention and graph inductive bias into a unified model. Then, as a further step, we show that the hybrid multi-layer model can be simplified to a one-layer counterpart that can, in principle, significantly enhance the computational efficiency without sacrificing the representation power. 

\subsection{A Hybrid Model Backbone}\label{sec-theory-hybrid}

Denote by $\mathbf Z^{(k)} = [\mathbf z_u^{(k)}]_{u=1}^N \in \mathbb R^{N\times d}$ the stack of $N$ nodes' embeddings at the $k$-th layer. We consider generic message-passing networks, which can unify the layer-wise updating rules of common GNNs and Transformers as a propagation layer with self-loop connection (a.k.a. residual link):
\begin{equation}\label{eqn-update}
    \mathbf Z^{(k+1)} = \mathbf P^{(k)} \mathbf Z^{(k)} \mathbf W^{(k)} + \beta \mathbf Z^{(k)} \mathbf W^{(k)},
\end{equation}
where $\beta\geq 0$ is a weight on the self-loop path, $\mathbf P^{(k)}=[p_{uv}^{(k)}]_{N\times N}$ denotes the propagation matrix at the $k$-th layer and $\mathbf W^{(k)} \in \mathbb R^{d\times d}$ denotes the layer-specific trainable weight matrix for feature transformation.
%and its Laplacian matrix $\mathbf \Delta^{(k)} = \tilde{\mathbf D}^{(k)} - \mathbf P^{(k)}$. 
For GNNs, the propagation matrix is commonly instantiated as a fixed sparse matrix, e.g., the normalized graph adjacency. For Transformers, $\mathbf P^{(k)}$ becomes a layer-specific dense attention matrix computed by $\mathbf Z^{(k)}$. 

\textbf{Propagation Layers as Optimization Dynamics.} The following theorem shows that, under mild conditions, the updating rule defined by Eqn.~\ref{eqn-update} is essentially an optimization step on a regularized energy that promotes a certain smoothness effect for graph signal denoising.
\begin{theorem}\label{thm-trans-opt}
    For any propagation matrix $\mathbf P^{(k)} = [p_{uv}^{(k)}]_{N\times N}$ and
    %$\beta \geq \max\{0, -\lambda\}$, where $\lambda$ is the smallest singular value of $\mathbf P^{(k)}$, 
    symmetric weight matrix $\mathbf W^{(k)}$, Eqn.~\ref{eqn-update} is a gradient descent step with step size $\frac{1}{2}$ for the optimization problem w.r.t. the quadratic energy: $E(\mathbf Z; \mathbf Z^{(k)}, \mathbf P^{(k)}, \mathbf W^{(k)})\triangleq$
    % \begin{equation}\label{eqn-energytrans}
    %      \min_{\mathbf Z}~~ \mbox{tr}\left( \mathbf Z^\top \mathbf P^{(k)} \mathbf Z\right ) + \left \| \mathbf Z - (\beta \mathbf I+ \tilde{\mathbf D}^{(k)}) \mathbf Z^{(k)} \right \|_{\mathcal F}^2,
    % \end{equation}
    \begin{equation}\label{eqn-energytrans}
         \sum_{u, v} p_{uv}^{(k)} \left \|\mathbf z_u - \mathbf z_v\right \|^2_{\mathbf W^{(k)}} + \sum_{u}  \| \mathbf z_u -  ( \beta + d_u^{(k)}) \mathbf W^{(k)} \mathbf z_u^{(k)} \|_2^2,
    \end{equation}
    where $ d_{u}^{(k)} = \sum_{v=1}^N p_{uv}^{(k)}$ and the weighted vector norm is defined by $\|\mathbf x\|_{\mathbf W}^2 = \mathbf x^\top \mathbf W \mathbf x$.
\end{theorem}
\begin{proof}
    We denote by $\mathbf D^{(k)} = \mbox{diag}(\{d_{u}^{(k)}\}_{u=1}^N)$ and $\boldsymbol \Delta^{(k)} = \mathbf D^{(k)} - \mathbf P^{(k)}$. The first term in $E(\mathbf Z; \mathbf Z^{(k)}, \mathbf P^{(k)}, \mathbf W^{(k)})$ can be expressed as $\mbox{tr}(\mathbf Z^\top \boldsymbol \Delta^{(k)} \mathbf Z \mathbf W^{(k)})$ and its gradient w.r.t. $\mathbf Z$ can be computed by
    \begin{equation}
        \frac{\partial \mbox{tr}(\mathbf Z^\top \boldsymbol \Delta^{(k)} \mathbf Z \mathbf W^{(k)})}{\partial \mathbf Z} = \boldsymbol \Delta^{(k)} \mathbf Z \cdot \left (\mathbf W^{(k)} + (\mathbf W^{(k)})^\top \right ).
    \end{equation}
    Given the symmetric property of $\mathbf W^{(k)}$, we have the gradient of $E(\mathbf Z; \mathbf Z^{(k)}, \mathbf P^{(k)}, \mathbf W^{(k)})$ evaluated at the point $\mathbf Z = \mathbf Z^{(k)}$:
    \begin{equation}
    \begin{split}
    & \left. \frac{\partial E(\mathbf Z; \mathbf Z^{(k)}, \mathbf P^{(k)}, \mathbf W^{(k)})}{\partial \mathbf Z} \right |_{\mathbf Z = \mathbf Z^{(k)}} \\
    = & 2  \boldsymbol\Delta^{(k)}  \mathbf Z^{(k)} \mathbf W^{(k)} + 2 \left [ \mathbf Z^{(k)} - (\beta \mathbf I +  \mathbf D^{(k)} )\mathbf Z^{(k)}\mathbf W^{(k)}  \right ] \\
    = & 2 \mathbf Z^{(k)} - 2\beta \mathbf Z^{(k)} \mathbf W^{(k)} - 2  \mathbf P^{(k)} \mathbf Z^{(k)} \mathbf W^{(k)},
    \end{split}
    \end{equation}
    where $\mathbf I$ denotes the $N\times N$ identity matrix. Using gradient descent with step size $\frac{1}{2}$ to minimize $E(\mathbf Z; \mathbf Z^{(k)}, \mathbf P^{(k)}, \mathbf W^{(k)})$ at the current layer yields an updating rule:
    \begin{equation}\label{eqn-thm1-proof1}
    \begin{split}
    \mathbf Z^{(k+1)} & = \mathbf Z^{(k)} - \frac{1}{2} \left. \frac{\partial E(\mathbf Z; \mathbf Z^{(k)}, \mathbf P^{(k)}, \mathbf W^{(k)})}{\partial \mathbf Z} \right |_{\mathbf Z = \mathbf Z^{(k)}}  \\
        & = \mathbf P^{(k)} \mathbf Z^{(k)} \mathbf W^{(k)} + \beta \mathbf Z^{(k)} \mathbf W^{(k)}.
    \end{split}
    \end{equation}
    We thus conclude the proof for the theorem.
    % \update{The convergence theorem of gradient descent implies that the above iteration converges if $\frac{1}{2}\leq \frac{1}{2(1-\beta) - 2\lambda}$. The latter is satisfied given the condition of the theorem.}
\end{proof}
The assumption of symmetric $\mathbf W^{(k)}$ can, to some extent, limit the applicability of this theorem, whereas, as we show later, the conclusion can be generalized to the case involving arbitrary $\mathbf W^{(k)}\in \mathbb R^{d\times d}$.
Now we discuss the implications of Theorem~\ref{thm-trans-opt}. The first term of Eqn.~\ref{eqn-energytrans} can be written as 
%$\sum_{u, v} p_{uv}^{(k)}\|\mathbf z_u^{(k)}-\mathbf z_v^{(k)}\|_2^2$, 
\begin{equation}
    \sum_{u, v} p_{uv}^{(k)} \left \|\mathbf z_u - \mathbf z_v\right \|^2_{\mathbf W^{(k)}}=\sum_{u,v} p_{uv}^{(k)} (\mathbf z_u - \mathbf z_v)^\top \mathbf W^{(k)} (\mathbf z_u - \mathbf z_v),
\end{equation}
which can be considered as generalization of the Dirichlet energy~\cite{reg-bernhard2005} defined over a discrete space of $N$ nodes where the pairwise distance between any node pair $(u, v)$ is given by $p_{uv}^{(k)}$ and the signal smoothness is measured through a weighted space $\|\cdot \|_{\mathbf W^{(k)}}$. The second term of Eqn.~\ref{eqn-energytrans} aggregates the square distance between the updated node embedding $\mathbf z_u$ and the last-layer embedding $\mathbf z_u^{(k)}$ after transformation of $(\beta + d_u^{(k)}) \mathbf W^{(k)}$. Overall, the objective of Eqn.~\ref{eqn-energytrans} formulates a graph signal denoising problem defined over $N$ nodes in a system that aims at smoothing the node embeddings via two-fold regularization effects~\cite{globallocal-2003} (as illustrated in Fig.~\ref{fig:theory}(a)): the first term penalizes the global smoothness among node embeddings through the proximity defined by $\mathbf P^{(k)}$; the second term penalizes the change of node embeddings from the ones prior to the propagation.

The theorem reveals that while the layer-wise updating rule adopted by either GNNs or Transformers can be unified as a descent step on a regularized energy, these two model backbones contribute to obviously different smoothness effects. For GNNs that use graph adjacency as the propagation matrix, in which situation $p_{uv}^{(k)} = 0$ for $(u, v)$'s that are disconnected in the graph, the energy only enforces global smoothness over neighboring nodes in the graph. In contrast, Transformers using all-pair attention induce the energy regularizing the global smoothness over arbitrary node pairs. The latter breaks the restriction of observed graphs and can facilitate leveraging the unobserved interactions for better representations. On the other hand, all-pair attention discards the input graph, which can play a useful inductive bias role in learning informative representations (especially when the observed structures strongly correlate with downstream labels). In light of the analysis, we next consider a hybrid propagation layer that synthesizes the effect of both models.

\textbf{A Hybrid Model Backbone.} We define a model backbone with the layer-wise updating rule comprised of three terms:
% \begin{equation}\label{eqn-propagation}
%     \mathbf P^{(k)} = (1-\alpha)\mathbf P^{(k)}_{\mathcal A} + \alpha \mathbf P_{\mathcal G},
% \end{equation}
\begin{equation}\label{eqn-propagation}
    \mathbf Z^{(k+1)} = (1 - \alpha)\mathbf P^{(k)}_A \mathbf Z^{(k)} \mathbf W^{(k)} + \alpha\mathbf P_G \mathbf Z^{(k)} \mathbf W^{(k)}  + \beta \mathbf Z^{(k)} \mathbf W^{(k)},
\end{equation}
where $\mathbf P^{(k)}_A$ is an all-pair attention-based propagation matrix specific to the $k$-th layer, $\mathbf P_G$ is a sparse graph-based propagation matrix (associated with input graph $\mathcal G$), and $0\leq\alpha< 1$ is a weight. We assume $\mathbf P^{(k)}_A = [c_{uv}^{(k)}]_{N\times N}$ and $\mathbf P_G = [w_{uv}]_{N\times N}$. Particularly, the hybrid model can be treated as an extension of Eqn.~\ref{eqn-update} where $\mathbf P^{(k)} = (1 - \alpha) \mathbf P^{(k)}_A + \alpha \mathbf P^{(k)}_G$ and specifically
\begin{equation}
    \begin{aligned}
    p_{uv}^{(k)}=\left\{ 
    \begin{array}{ll}
         & (1-\alpha)c_{uv}^{(k)} + \alpha w_{uv}, \quad \mbox{if} \; (u,v)\in \mathcal E  \\
         & (1-\alpha)c_{uv}^{(k)}, \quad \mbox{if} \; (u,v)\notin \mathcal E.
    \end{array}
    \right. 
    \end{aligned}
\end{equation}
We can extend the result of Theorem~\ref{thm-trans-opt} and naturally derive the regularized energy optimized by the hybrid model.
\begin{corollary}\label{coro-hybrid-energy}
    For any attention-based propagation matrix $\mathbf P^{(k)}_A$ and graph-based propagation matrix $\mathbf P_G$, if $\mathbf W^{(k)}$ is a symmetric matrix, then Eqn.~\ref{eqn-propagation} is a gradient descent step with step size $\frac{1}{2}$ for the optimization problem w.r.t. the quadratic energy $E(\mathbf Z; \mathbf Z^{(k)}, \mathbf P^{(k)}_A, \mathbf P_G, \mathbf W^{(k)})\triangleq$
    \begin{equation}\label{eqn-energytrans2}
    \begin{split}
        & \sum_{u, v} \left [ (1 - \alpha) c_{uv}^{(k)} \left \|\mathbf z_u - \mathbf z_v\right \|^2_{\mathbf W^{(k)}} + \alpha w_{uv} \left \|\mathbf z_u - \mathbf z_v\right \|^2_{\mathbf W^{(k)}} \right ] \\
        & + \sum_{u} \left \| \mathbf z_u -  \left ( \beta + (1 - \alpha) \tilde d_u^{(k)}  + \alpha \overline d_u \right ) \mathbf W^{(k)} \mathbf z_u^{(k)} \right \|_2^2,
    \end{split}
    \end{equation}
    where $\tilde d_u^{(k)} = \sum_{v=1}^N c_{uv}^{(k)}$ and $\overline d_u= \sum_{v=1}^N w_{uv}$.
\end{corollary}
% \begin{equation}\label{eqn-energyours}
% \begin{split}
%     & \min_{\mathbf Z} ~~
%     (1-\alpha)\mbox{tr}\left( \mathbf Z \mathbf P^{(k)}_{\mathcal A} \mathbf Z\right ) + \alpha \mbox{tr}\left( \mathbf Z \mathbf P_{\mathcal G} \mathbf Z\right ) \\
%     & + \left \| \mathbf Z - \left (\beta \mathbf I+ (1-\alpha)\tilde{\mathbf D}^{(k)}_{\mathcal A} + \alpha \tilde{\mathbf D}_{\mathcal G} \right ) \mathbf Z^{(k)} \right \|_{\mathcal F}^2,
% \end{split}
% \end{equation}
% where $\tilde{\mathbf D}^{(k)}_{\mathcal A}$ (resp. $\tilde{\mathbf D}_{\mathcal G}$) denotes the diagonal in-degree matrix associated with $\mathbf P^{(k)}_{\mathcal A}$ (resp. $\mathbf P_{\mathcal G}$).
\begin{proof}
    The proof for this corollary can be adapted by Theorem~\ref{thm-trans-opt} with the similar reasoning line.
\end{proof}
The hybrid model is capable of accommodating observed structural information and in the meanwhile capturing unobserved interactions beyond input graphs. Such an architectural design incorporates the graph inductive bias into the vanilla Transformer and is adopted by state-the-of-art Transformers on graphs, e.g., ~\cite{nodeformer,graphgps,wu2023difformer}, that show superior performance in different tasks of graph representation learning. 

\textbf{Generalization to Asymmetric Weight Matrix.} The above analysis assumes the weight matrix $\mathbf W^{(k)}$ to be symmetric, which may limit the applicability of the conclusions since in common neural networks the weight matrix can potentially take any value in the entire $\mathbb R^{d\times d}$. In our context, it can be difficult to directly analyze the case of asymmetric $\mathbf W^{(k)}$ and derive any closed form of the energy. However, the following proposition allows us to generalize the conclusion of Theorem~\ref{thm-trans-opt} to arbitrary $\mathbf W^{(k)} \in \mathbb R^{d\times d}$.
\begin{proposition}\label{prop-w-sym}
    For any weight matrix $\mathbf W^{(k)}\in \mathbb R^{d\times d}$, there exists a symmetric matrix $\tilde{\mathbf W}^{(k)}\in \mathbb R^{d\times d}$ such that the updated embeddings $\tilde{\mathbf Z}^{(k+1)} = \mathbf P^{(k)} \mathbf Z^{(k)} \tilde{\mathbf W}^{(k)} + \beta \mathbf Z^{(k)} \tilde{\mathbf W}^{(k)}$ yield $\|\tilde{\mathbf Z}^{(k+1)} - \mathbf Z^{(k+1)}\|<\epsilon$, $\forall \epsilon>0$. 
\end{proposition}
\begin{proof}
    We extend the proof of Theorem 9 in \cite{yang2021implicit} to our case. The updating rule considered in Lemma 26 of \cite{yang2021implicit} can be replaced by our updating rule $\mathbf Z^{(k+1)} = \mathbf P^{(k)} \mathbf Z^{(k)} \mathbf W^{(k)} + \beta \mathbf Z^{(k)} \mathbf W^{(k)}$, so that the updated embeddings $\mathbf Z^{(k+1)}$ are continuous w.r.t. $\mathbf W^{(k)}$. Then similar to Theorem 27 of \cite{yang2021implicit}, we can prove that for any weight matrix $\mathbf W^{(k)}\in \mathbb R^{d\times d}$, there exist $\tilde{\mathbf Z}^{(k+1)} \in \mathbb C^{N\times d}$, right-invertible $\mathbf T \in \mathbb C^{d'\times d}$ and herimitia $\tilde{\mathbf W}^{(k)}\in \mathbb C^{N\times d'}$ such that $\tilde{\mathbf Z}^{(k+1)} = \mathbf P^{(k)} \mathbf Z^{(k)} \tilde{\mathbf W}^{(k)} + \beta \mathbf Z^{(k)} \tilde{\mathbf W}^{(k)}$ and
    \begin{equation}
        \|\tilde{\mathbf Z}^{(k+1)} - \mathbf Z^{(k+1)}\|<\epsilon, \forall \epsilon>0.
    \end{equation}
    By assuming $\mathbf T \in \mathbb C^{d\times d}$, we prove the conclusion in the complex domain. Then we can use the same technique as the proof after Theorem 27 in Appendix E.3 of \cite{yang2021implicit} to generalize the conclusion from the complex domain to the real domain.
\end{proof}
This suggests that for any propagation layer (as defined by Eqn.~\ref{eqn-update}) with $\mathbf W^{(k)}\in \mathbb R^{d\times d}$ that is even asymmetric, we can find a surrogate matrix $\tilde{\mathbf W}^{(k)}$ that is symmetric, such that the latter produces the node embeddings which can be arbitrarily close to the ones produced by $\mathbf W^{(k)}$. Pushing further, one-layer updates of the message passing model corresponds to a descent step on the regularized energy $E(\mathbf Z; \mathbf Z^{(k)}, \mathbf P^{(k)}, \tilde{\mathbf W}^{(k)})$.

\subsection{Reduction from Multi-Layer to One-Layer}\label{sec-theory-onelayer}

The analysis so far targets the updates on embeddings of one propagation layer, yet the model practically used for computing representations often stacks multiple propagation layers. While using deep propagation may endow the model with desired expressivity, it also increases the model complexity and hinders its scalability to large graphs. We next zoom in on whether using multi-layer propagation is a necessary condition for satisfactory expressiveness for learning representations. On top of this, the analysis suggests a potential way to simplify the Transformer architecture for learning on large graphs. 

\textbf{Multi-Layer v.s. One-Layer Models.} The analysis in Sec.~\ref{sec-theory-onelayer} reveals the equivalence between the embedding updates of one propagation layer and one-step gradient descent on the regularized energy. Notice that since the attention matrix $\mathbf P^{(k)}_{\mathcal A}$ (dependent on $\mathbf Z^{(k)}$) and the feature transformation $\mathbf W^{(k)}$ vary at different layers, the energy objective optimized by the model (Eqn.~\ref{eqn-update} or Eqn.~\ref{eqn-propagation}) is also specific to each layer. In this regard, the multi-layer model, which is commonly adopted by existing Transformers, can be seen as a cascade of descent steps on layer-dependent energy objectives. From this viewpoint, there potentially exists certain redundancy in the optimization process for graph signal processing, since the descent steps of different layers pursue different targets and may interfere with each other. 
To resolve this issue, we introduce the next theorem that further suggests a principled way to simplify the Transformer model, and particularly, we can construct a single-layer model that yields the same denoising effect as the multi-layer counterpart.
\begin{theorem}\label{thm-trans-opt-equi}
    For any $K$-layer model (where $K$ is an arbitrary positive integer) whose layer-wise updating rule is defined by Eqn.~\ref{eqn-propagation} producing the output embeddings $\mathbf Z^{(K)}$, there exists a (sparse) graph-based propagation matrix $\mathbf P^*_G=[w^*_{uv}]_{N\times N}$, a dense attention-based propagation matrix $\mathbf P^*_A = [c^*_{uv}]_{N\times N}$, and a symmetric weight matrix $\mathbf W^*\in \mathbb R^{d\times d}$ such that one gradient descent step for optimization (from the initial point $\mathbf Z^{(0)} = [\mathbf z_u^{(0)}]_{u=1}^N$) w.r.t. the surrogate energy
    $E^*(\mathbf Z; \mathbf Z^{(k)}, \mathbf P^*_A, \mathbf P^*_G, \mathbf W^*)\triangleq$
    \begin{equation}\label{eqn-energyeqv}
    \begin{split}
        & \sum_{u, v} \left [ (1 - \alpha) c_{uv}^* \left \|\mathbf z_u - \mathbf z_v\right \|^2_{\mathbf W^*} + \alpha w_{uv}^* \left \|\mathbf z_u - \mathbf z_v\right \|^2_{\mathbf W^*} \right ] \\
        & + \sum_{u} \left \| \mathbf z_u -  \left ( (1 - \alpha) \tilde d_u^{(k)}  + \alpha \overline d_u  \right ) \mathbf W^*\mathbf z_u^{(k)} \right \|_2^2,
    \end{split}
    \end{equation}
    where $\tilde d_u^{(k)} = \sum_{v=1}^N c_{uv}^{(k)}$ and $\overline d_u= \sum_{v=1}^N w_{uv}$, yields node embeddings $\mathbf Z^*$ satisfying $\|\mathbf Z^* - \mathbf Z^{(K)}\|<\epsilon$, $\forall \epsilon >0$. 
%     \begin{equation}\label{eqn-energyeqv}
% \begin{split}
%     & \min_{\mathbf Z}~~ 
%     (1-\alpha)\mbox{tr}\left( \mathbf Z \mathbf P^*_{\mathcal A} \mathbf Z\right ) + \alpha \mbox{tr}\left( \mathbf Z \mathbf P^*_{\mathcal G} \mathbf Z\right ) \\
%     & + \left \| \mathbf Z - \left (\mathbf I+ (1-\alpha)\tilde{\mathbf D}^*_{\mathcal A} + \alpha \tilde{\mathbf D}_{\mathcal G}^* \right ) \mathbf Z^{(0)} \right \|_{\mathcal F}^2,
% \end{split}
% \end{equation}
% \begin{equation}\label{eqn-energyeqv}
% \begin{split}
%     & \min_{\mathbf Z}~~ 
%     (1-\alpha)\mbox{tr}\left( \mathbf Z \mathbf P^*_{\mathcal A} \mathbf Z\right ) + \alpha \mbox{tr}\left( \mathbf Z \mathbf P^*_{\mathcal G} \mathbf Z\right ) \\
%     & + \left \| \mathbf Z - \left (\mathbf I+ (1-\alpha)\tilde{\mathbf D}^*_{\mathcal A} + \alpha \tilde{\mathbf D}_{\mathcal G}^* \right ) \mathbf Z^{(0)} \right \|_{\mathcal F}^2,
% \end{split}
% \end{equation}
% where $\tilde{\mathbf D}^*_{\mathcal A}$ (resp. $\tilde{\mathbf D}^*_{\mathcal G}$) denotes the diagonal in-degree matrix associated with $\mathbf P^*_{\mathcal A}$ (resp. $\mathbf P^*_{\mathcal G}$), yields the output embeddings $\mathbf Z^{(K)}$ of the $K$-layer hybrid model.
\end{theorem}
\begin{proof}
    By definition, the $K$-layer model is comprised of $K$ feed-forward updating layers each of which adopts the propagation $\mathbf P^{(k)} = (1-\alpha)\mathbf P^{(k)}_A + \alpha \mathbf P_G$ (where notice that $\mathbf P^{(k)}_A$ is computed by $\mathbf Z^{(k)}$) to update the node embeddings from $\mathbf Z^{(k)}$ to $\mathbf Z^{(k+1)}$. The feed-forward computation yields a sequence of results $\mathbf Z^{(0)}, \mathbf P^{(0)}, \mathbf Z^{(1)}, \mathbf P^{(1)}, \cdots, \mathbf Z^{(K)}$. We introduce an augmented propagation matrix $\overline{\mathbf P}^{(k)}$ and rewrite Eqn.~\ref{eqn-propagation} as
    \begin{equation}
        \mathbf Z^{(k+1)} = (\mathbf P^{(k)} + \beta \mathbf I) \mathbf Z^{(k)} \mathbf W^{(k)} = \overline{\mathbf P}^{(k)} \mathbf Z^{(k)} \mathbf W^{(k)}.
    \end{equation}
    By stacking $K$ layers of propagation, we can denote the output embeddings as
    \begin{equation}\label{eqn-k-prop}
    \begin{split}
        \mathbf Z^{(K)} & = \overline{\mathbf P}^{(K-1)} \cdots \overline{\mathbf P}^{(0)} \mathbf Z^{(0)} \mathbf W^{(0)} \cdots \mathbf W^{(K-1)} \\
        & = \overline{\mathbf P}^* \mathbf Z^{(0)} \tilde{\mathbf W}^*.
    \end{split}
    \end{equation}
    We next conclude the proof by construction. Assume 
    \begin{equation}
        \mathbf P^*_A = \frac{1}{1-\alpha} \left (\overline{\mathbf P}^* - \alpha \mathbf P^*_G \right ),
    \end{equation}
    where $\mathbf P^*_G$ is a sparse matrix associated to $\mathcal G$ which can be arbitrarily given, e.g., $\mathbf P^*_G = \mathbf P_G^K$. The latter becomes the $K$-order (normalized) adjacency matrix if $\mathbf P_G$ is the (normalized) adjacency matrix. Then we consider the optimization problem w.r.t. the energy $E^*(\mathbf Z; \mathbf Z^{(0)}, \mathbf P^*_A, \mathbf P^*_G, \mathbf W^*)$ defined by Eqn.~\ref{eqn-energyeqv}. From the initial point $\mathbf Z^{(0)} = [\mathbf z_u^{(0)}]_{u=1}^N$, the gradient w.r.t. $\mathbf Z$ can be evaluated by
    \begin{equation}
    \begin{split}
    & \left. \frac{\partial E^*(\mathbf Z; \mathbf Z^{(0)}, \mathbf P^*_A, \mathbf P^*_G, \mathbf W^*)}{\partial \mathbf Z} \right |_{\mathbf Z = \mathbf Z^{(0)}} \\
    = & 2 (1-\alpha) (\mathbf D^*_A - \mathbf P^*_A) \mathbf Z^{(0)} \mathbf W^*  + \alpha (\mathbf D^*_G - \mathbf P^*_G) \mathbf Z^{(0)} \mathbf W^* \\
    & + 2 \left [ \mathbf Z^{(0)} - \left ( (1 - \alpha)\mathbf D^*_A + \alpha \mathbf D^*_G \right ) \mathbf Z^{(0)}\mathbf W^*  \right ] \\
    = &  2 \mathbf Z^{(0)} - (1-\alpha) \mathbf P^*_A \mathbf Z^{(0)} \mathbf W^* - 2 \alpha \mathbf P^*_G \mathbf Z^{(0)}\mathbf W^* .
    \end{split}
    \end{equation}
    Inserting the gradient into the updating and using step size $\frac{1}{2}$ will induce the descent step:
    \begin{equation}
    % \begin{split}
        \mathbf Z^{(0)} - \frac{1}{2} \left. \frac{\partial E^*(\mathbf Z; \mathbf Z^{(0)},\mathbf P^*_A, \mathbf P^*_G, \mathbf W^*)}{\partial \mathbf Z} \right |_{\mathbf Z = \mathbf Z^{(0)}} 
        = \overline{\mathbf P}^* \mathbf Z^{(0)} \mathbf W^*.
    % \end{split}
    \end{equation}
    We thus have shown that solving the optimization problem w.r.t. Eqn.~\ref{eqn-energyeqv} via one-step gradient descent induces the output embeddings $\mathbf Z^*=\overline{\mathbf P}^* \mathbf Z^{(0)} \mathbf W^* $. According to Eqn.~\ref{eqn-k-prop} and using Proposition~\ref{prop-w-sym}, we have $\|\mathbf Z^* - \mathbf Z^{(K)}\|<\epsilon$, $\forall \epsilon >0$.
\end{proof}
This theorem indicates that for any multi-layer model whose propagation layers are defined by Eqn.~\ref{eqn-propagation}, its induced optimization trajectories composed of multiple gradient descent steps (updating node embeddings from $\mathbf Z^{(0)}$ to $\mathbf Z^{(K)}$) can be reduced to one-step gradient descent on the energy objective of Eqn.~\ref{eqn-energyeqv}. 
The latter produces the node embeddings that have negligibly small approximation error compared to the ones yielded by the multi-layer model. Based on Theorem~\ref{thm-trans-opt-equi} and extending the analysis of Sec.~\ref{sec-theory-hybrid}, we can arrive at the following result that suggests a simplified one-layer model.
\begin{corollary}\label{coro-one-layer}
    For any $K$-layer model whose layer-wise updating rule is defined by Eqn.~\ref{eqn-propagation}, there exists a graph-based propagation matrix $\mathbf P_G^*$, an attention-based propagation matrix $\mathbf P_A^{(k)}$, and a weight matrix $\mathbf W^*\in \mathbb R^{d\times d}$, such that the one-layer model with the updating rule
    \begin{equation}\label{eqn-onelayer-model}
    \mathbf Z^* = (1-\alpha) \mathbf P^*_A \mathbf Z^{(0)} \mathbf W^* + \alpha \mathbf P^*_G \mathbf Z^{(0)} \mathbf W^*,
    \end{equation}
    yields the equivalent result with up to negligible approximation error $\|\mathbf Z^* - \mathbf Z^{(K)}\|<\epsilon$, $\forall \epsilon >0$.
\end{corollary}
\begin{proof}
    Theorem~\ref{thm-trans-opt-equi} indicates that there exists $\mathbf P_G^*$, $\mathbf P_A^{(k)}$, and $\tilde{\mathbf W}^*\in \mathbb R^{d\times d}$ that is symmetric so that one-step gradient descent on $E^*(\mathbf Z; \mathbf Z^{(0)},\mathbf P^*_A, \mathbf P^*_G, \tilde{\mathbf W}^*)$ yields the node embeddings $\tilde{\mathbf Z}^*$ with the approximation error $\|\tilde{\mathbf Z}^* - \mathbf Z^{(K)}\|< \frac{\epsilon}{2}$. Furthermore, similar to the reasoning line of Theorem~\ref{thm-trans-opt}, we can show that one-step gradient descent on $E^*(\mathbf Z; \mathbf Z^{(0)},\mathbf P^*_A, \mathbf P^*_G, \tilde{\mathbf W}^*)$ as defined by Eqn.~\ref{eqn-energyeqv} induces the updating rule
    \begin{equation}
        \tilde{\mathbf Z}^* = (1-\alpha) \mathbf P^*_A \mathbf Z^{(0)} \tilde{\mathbf W}^* + \alpha \mathbf P^*_G \mathbf Z^{(0)} \tilde{\mathbf W}^*.
    \end{equation}
    By comparing $\tilde{\mathbf Z}^*$ and $\mathbf Z^*$ in Eqn.~\ref{eqn-onelayer-model} and applying Proposition~\ref{prop-w-sym}, we have the result $\|\tilde{\mathbf Z}^* - \mathbf Z^*\|<\frac{\epsilon}{2}$. Then the corollary can be obtained using the triangle inequality.
\end{proof}
In this sense, the single-layer model produces nearly the same denoising effect (where the approximation error can be arbitrarily small) as the multi-layer model, and the output node embeddings exhibit in the same way to leverage the global information from other nodes. A more intuitive illustration is provided in Fig.~\ref{fig:theory}(b).

\subsection{Implications for Model Designs}\label{sec-theory-summary}

As conclusion of the analysis in this section, we frame our main results as the following two statements and shed more insights on practical model designs, especially for building powerful and scalable Transformers on large graphs.

\begin{itemize}
    \item \textbf{Statement 1.} The layer-wise propagation rule of the generic message passing model (Eqn.~\ref{eqn-update}) is equivalent to an optimization step for the objective of graph signal denoising (Theorem~\ref{thm-trans-opt}). The latter regularizes two-fold smoothness (as illustrated by Fig.~\ref{fig:theory}(a)) and this principled viewpoint further induces a hybrid propagation layer synthesizing the advantage of GNNs and all-pair global attention (Corollary~\ref{coro-hybrid-energy}).
    \item \textbf{Statement 2.} From the perspective of graph signal denoising, stacking multiple propagation layers is not a necessity for achieving the desired expressiveness, since there exists a one-layer model producing the equivalent embeddings as the multi-layer model (Theorem~\ref{thm-trans-opt-equi} and Corollary~\ref{coro-one-layer}). Moreover, as compared in Fig.~\ref{fig:theory}(b), the multi-layer model optimizes different objectives at each layer, while the one-layer model contributes to a steepest descent step on a single fixed objective, reducing the potential redundancy. 
\end{itemize}
Therefore, from this standpoint, the multi-layer model can be simplified to the single-layer model without sacrificing the expressiveness for representation. This result enlightens a potential way to build efficient and powerful Transformers on large graphs, as will be exemplified in the next section.

\begin{figure*}[t!]
  \centering
  \includegraphics[width=0.95\linewidth]{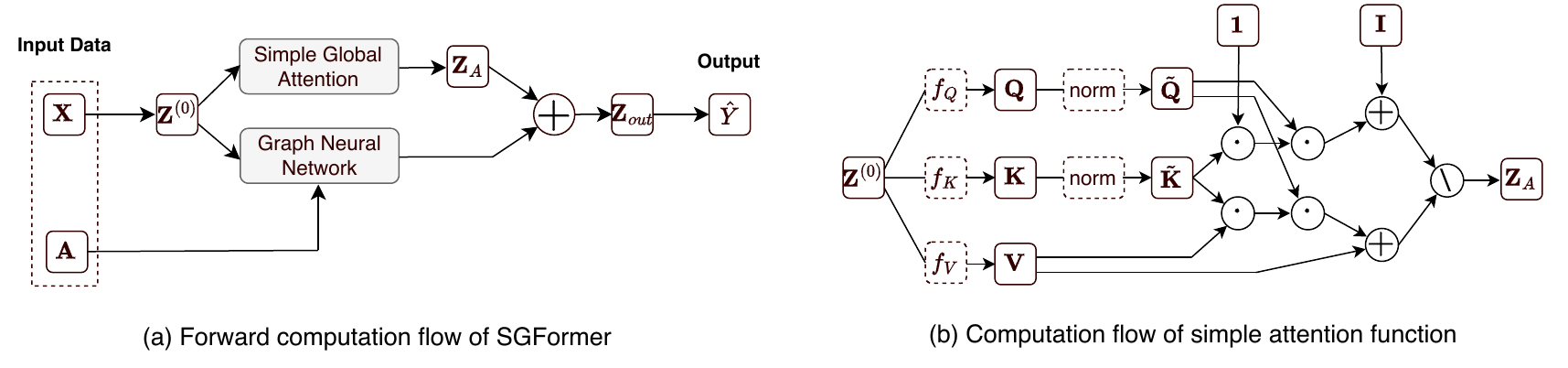}
  % \vspace{-5pt}
  \caption{(a) Data flow of SGFormer. The input data entails node features $\mathbf X$ and graph adjacency $\mathbf A$. %For large graphs, we need to use mini-batch sampling that randomly partitions the input graph into mini-batches with smaller sizes. Each mini-batch is composed of the features of the nodes within this mini-batch $\mathbf X_m$ and the local graph adjacency $\mathbf A_m$ (one can also use neighbor sampling as an alternative). The mini-batch data $(\mathbf X_m, \mathbf A_m)$ (for large graphs) or the whole graph data $(\mathbf X, \mathbf A)$ (for small graphs) will be fed into the SGFormer 
  SGFormer is comprised of a single-layer global attention and a GNN network. The model outputs node representations for final prediction. (b) Computation flow of the simple attention function utilized by SGFormer which accommondates all-pair influence among $N$ nodes for computing the updated embeddings within $O(N)$ complexity.}
  \label{fig:model}
  % \vspace{-5pt}
\end{figure*}

\section{Proposed Model}\label{sec-model}

In this section, we introduce our model, referred as Simplified Graph Transformer (SGFormer), under the guidance of our theoretical results in Sec.~\ref{sec-theory} (as distilled in Sec.~\ref{sec-theory-summary}). Overall, the architectural design  adopts the hybrid model in Eqn.~\ref{eqn-onelayer-model} and follows the Occam's Razor principle for specific instantiations.
Particularly, SGFormer only requires $O(N)$ complexity for accommodating the all-pair interactions and computing $N$ nodes' representations. This is achieved by a simple attention function which has advanced computational efficiency and is free from any approximation scheme. Apart from the scalability advantage, the light-weighted architecture endows SGFormer with desired capability for learning on large graphs with limited labels.

\subsection{Model Design}\label{sec-model-design}

% With the input of node features $\mathbf X = [\mathbf x_u]_{u=1}^N$ and graph structures $\mathbf A = [a_{uv}]_{N\times N}]$, the model aims at producing effective node embeddings $\mathbf Z = [\mathbf z_u]_{u=1}^N$ as representations used for downstream prediction. The ideal node representations are expected to be informative enough, i.e., encoding the useful information within input data. In the meanwhile, the model needs to be equipped with desired efficiency and scalability such that it can scale to graphs of large sizes. We next illustrate the details of our proposed model whose design philosophy aims at simplifying the Transformer architecture.

We first use a neural layer to map input features $\mathbf X = [\mathbf x_u]_{u=1}^N$ to node embeddings in the latent space, i.e., $\mathbf Z^{(0)} = f_I(\mathbf X)$ where $f_I$ can be a shallow (e.g., one-layer) MLP. Then based on the result of Sec.~\ref{sec-theory-onelayer}, particularly the single-layer model presented in Corollary~\ref{coro-one-layer}, we consider the following hybrid architecture for updating the embeddings:
\begin{equation}\label{eqn-model-imple}
    \mathbf Z_{out} = (1-\alpha) \mbox{AN}(\mathbf Z^{(0)}) + \alpha \mbox{GN}(\mathbf Z^{(0)}, \mathbf A),
\end{equation}
where $0\leq \alpha < 1$ again is a hyper-parameter, and $\mbox{AN}$ and $\mbox{GN}$ denote a global attention network and a graph-based propagation network, respectively. Then the node representations $\mathbf Z_{out}$ are fed into an output neural layer for prediction $\hat Y = f_O(\mathbf Z_{out})$, where $f_O$ is a fully-connected layer in our implementation. We next delve into the detailed instantiations of $\mbox{AN}$ and $\mbox{GN}$.

\textbf{Simple Global Attention Network.} %The global interactions in our model are captured by an all-pair attention unit that computes pair-wise influence between arbitrary node pairs. Unlike other Transformers, which often require multiple attention layers for desired capacity, we found that a single-layer global attention is sufficient. This is because through one-layer propagation over a densely connected attention graph, the information of each node can be adaptively propagated to arbitrary nodes within the batch. Therefore, despite its simplicity, our model has sufficient expressivity to capture implicit dependencies among arbitrary node pairs while significantly reducing computational overhead. 
There exist many potential choices for global attention functions as the instantiation of $\mbox{AN}(\mathbf Z^{(0)})$, e.g., the widely adopted Softmax attention that is originally used by \cite{transformer}. While the Softmax attention possesses provable expressivity~\cite{tmlr}, it requires $O(N^2)$ complexity for computing the all-pair attentions and updating the representations of $N$ nodes. This computational bottleneck hinders its scalability for large graphs. Alternatively, we introduce a simple attention function which can reduce the computational complexity to $O(N)$ and still accommodate all-pair interactions. Specifically, with the initial embeddings $\mathbf Z^{(0)}$ as input, we first use feature transformations $f_Q$, $f_K$ and $f_V$ to obtain the key, query and value matrices, respectively, as is done by common Transformers:
\begin{equation}
    \mathbf Q = f_Q(\mathbf Z^{(0)}), \quad \mathbf K = f_K(\mathbf Z^{(0)}), \quad \mathbf V = f_V(\mathbf Z^{(0)}), \\
\end{equation}
where $f_Q$, $f_K$ and $f_V$ are instantiated as a fully-connected layer in our implementation. Then we consider the attention function that computes the all-pair similarities:
\begin{equation}\label{eqn-new-attention}
    \overline{\mathbf C} = \mathbf I + \frac{1}{N} \left (  \frac{\mathbf Q}{\|\mathbf Q\|_{\mathcal F}} \right )\cdot \left ( \frac{\mathbf K}{\|\mathbf K\|_{\mathcal F}} \right )^\top,
\end{equation}
\begin{equation}\label{eqn-normalization}
    \mathbf C = \mbox{diag}^{-1}\left ( \overline{\mathbf C} \mathbf 1 \right ) \cdot \overline{\mathbf C},
\end{equation}
where $\mathbf 1$ is an $N$-dimensional all-one column vector. In Eqn.~\ref{eqn-new-attention} the scaling factor $\frac{1}{N}$ can improve the numerical stability and the addition of a self-loop can help to strengthen the role of central nodes. Eqn.~\ref{eqn-normalization} serves as row-normalization which is commonly used in existing attention designs. If one uses the attention matrix $\mathbf C$ to compute the updated embeddings, i.e., $\mathbf Z_{AN} = \mathbf C \mathbf V$, the computation requires $O(N^2)$ complexity, since the computation of the attention matrix (Eqn.~\ref{eqn-new-attention}) and the updated embeddings both needs the cost of $O(N^2)$. Notably, the simple attention function allows an alternative way for computing the updated embeddings via changing the order of matrix products. In specific, assume $\tilde{\mathbf Q} = \frac{\mathbf Q}{\|\mathbf Q\|_{\mathcal F}}$ and $\tilde{\mathbf K} = \frac{\mathbf K}{\|\mathbf K\|_{\mathcal F}}$, and we can rewrite the computation flow of the attention-based propagation:
\begin{equation}
     \mathbf N = \mathrm{diag}^{-1} \left(\mathbf{I} + \frac{1}{N} \tilde{\mathbf Q} (\tilde{\mathbf K}^\top \mathbf 1)  \right),
\end{equation}
\begin{equation}\label{eqn-attn-matrix-our}
    \mathbf Z_{AN} = \mathbf N \cdot \left[ \mathbf{V}+ \frac{1}{N} \tilde{\mathbf Q}( \tilde{\mathbf K}^\top \mathbf V) \right].
\end{equation}
One can verify through basic linear algebra that the result of Eqn.~\ref{eqn-attn-matrix-our} is equivalent to the one obtained by $\mathbf Z_{AN} = \mathbf C\mathbf V$ which explicitly computes the all-pair attention. In other words, while the computation flow of Eqn.~\ref{eqn-attn-matrix-our} does not compute the all-pair attention matrix, it still accommodates the all-pair interactions as the original attention. More importantly, the computation of Eqn.~\eqref{eqn-attn-matrix-our} can be achieved in $O(N)$ complexity, which is much more efficient than using the original computation flow. Therefore, such a simple global attention design reduces the quadratic complexity to $O(N)$ and in the meanwhile guarantee the expressivity for capturing all-pair interactions. 

\textbf{Graph-based Propagation Network.} For accommodating the prior information of the input graph $\mathcal G$, existing models tend to use positional encodings~\cite{graphgps}, edge regularization loss~\cite{nodeformer} or augmenting the Transformer layers with GNNs~\cite{graphtrans-neurips21}.  %such as positional encodings~\cite{}, graph regularization loss~\cite{}, edge embeddings~\cite{} or combining attention with GNN layers~\cite{}. In a general sense, the last strategy can be the most straightforward one that does not require any tricky design or complicate the model architectures. 
Here we resort to a simple-yet-effective scheme and implement $\mbox{GN}(\mathbf Z^{(0)}, \mathbf A)$ (in Eqn.~\ref{eqn-model-imple}) through a simple graph neural network (e.g., GCN~\cite{GCN-vallina}) that possesses good scalability for large graphs. We note that while the theoretical results in Sec.~\ref{sec-theory} suggest that using one-layer propagation can in principle achieve equivalent expressiveness as multi-layer propagation, it does not necessarily mean in practice the model has to be constrained to single-layer architectures. The main advantage of reducing the propagation layers lies in the improvement in computational efficiency, which is already achieved by SGFormer with the adoption of single-layer global attention since the global attention determines the computational overhead of Transformers. In this regard, one can still use shallow layers of GNNs in practice that are desirably efficient and require negligible extra costs. 

\textbf{Complexity Analysis.} Algorithm~\ref{alg:method} presents the feedforward computation and training process of SGFormer. The overall computational complexity of our model is $O(N+E)$, where $E = |\mathcal E|$, as the GN module requires $O(E)$. Due to the typical sparsity of graphs (i.e., $E\ll N^2$), our model can scale linearly w.r.t. graph sizes. Furthermore, with only single-layer global attention and simple GNN architectures, our model is fairly lightweight, enabling efficient training and inference.

\textbf{Scaling to Larger Graphs.} For larger graphs that even GCN cannot be trained on using full-batch processing with a single GPU, we can use the random mini-batch partitioning method utilized by \cite{nodeformer}, which we found works well and efficiently in practice. Specifically, we randomly shuffle all the nodes and partition the nodes into mini-batches with the size $B$. Then in each iteration, we feed one mini-batch (the input graph among these $B$ nodes are directly extracted by the subgraph of the original graph) into the model for loss computation on the training nodes within this mini-batch. This scheme incurs negligible additional costs during training and allows the model to scale to arbitrarily large graphs. Moreover, owing to the linear complexity w.r.t. node numbers required by SGFormer, we can employ large batch sizes (e.g., $B=0.1M$), which facilitate the model in capturing informative global interactions among nodes within each mini-batch. Apart from this simple scheme, our model is also compatible with other techniques such as neighbor sampling~\cite{graphsaint}, graph clustering~\cite{clustergcn-kdd19}, and historical embeddings~\cite{gnnautoscale}. These techniques may require extra time costs for training, and we leave exploration along this orthogonal direction for future works. Fig.~\ref{fig:model} presents the data flow of the proposed model.

\begin{algorithm}[t]
\caption{Feed-forward and Training of SGFormer.}
  \label{alg:method}
  % \small
\begin{algorithmic}[1]
  \STATE {\bfseries Input:} Node feature matrix $\mathbf{X}$, input graph adjacency matrix $\mathbf A$, labels of training nodes $Y_{tr}$, weight on graph-based propagation $\alpha$.
  \WHILE{not reaching the budget of training epochs}
  \STATE Encode input node features $\mathbf Z^{(0)} = f_I(\mathbf X)$;\\
  \STATE Compute query, key and value matrices $\mathbf Q = f_Q(\mathbf Z^{(0)})$, $\mathbf K=f_K(\mathbf Z^{(0)})$ and $\mathbf V = f_V(\mathbf Z^{(0)})$; \\
\STATE Compute normalization $\tilde{\mathbf Q} = \frac{\mathbf Q}{\|\mathbf Q\|_{\mathcal F}}$ and $\tilde{\mathbf K} = \frac{\mathbf K}{\|\mathbf K\|_{\mathcal F}}$; \\
\STATE Compute denominator of global attention $\mathbf N = \mathrm{diag}^{-1} \left(\mathbf{I} + \frac{1}{N} \tilde{\mathbf Q} (\tilde{\mathbf K}^\top \mathbf 1)  \right)$; \\
\STATE Compute updated embeddings by global attention $\mathbf Z_{AN} = \mathbf N \cdot \left[ \mathbf{V}+ \frac{1}{N} \tilde{\mathbf Q}( \tilde{\mathbf K}^\top \mathbf V) \right]$; \\
\STATE Compute final representations by graph-based propagation $\mathbf Z_{out} = (1-\alpha) \mathbf Z^{AN} + \alpha \mbox{GN}(\mathbf Z^{(0)}, \mathbf A)$; \\
\STATE Calculate predicted labels $\hat{Y} = f_O(\mathbf Z_{out})$; \\
\STATE Compute the supervised loss $\mathcal L$ from $\hat Y_{tr}$ and $Y_{tr}$; \\
\STATE Use $\mathcal L$ to update the trainable parameters;
    \ENDWHILE
\end{algorithmic}
\end{algorithm}

\subsection{Comparison with Existing Models}\label{sec-model-comp}

We next provide a more in-depth discussion comparing our model with prior art and illuminating its potential in wide application scenarios. Table~\ref{tab:comparison} presents a head-to-head comparison of current graph Transformers in terms of their architectures, expressivity, and scalability. Most existing Transformers have been developed and optimized for graph classification tasks on small graphs, while some recent works have focused on Transformers for node classification, where the challenge of scalability arises due to large graph sizes. 

\begin{table*}[t]
	\centering
 \caption{Comparison of (typical) graph Transformers w.r.t. required components (positional encodings, multi-layer attentions, augmented loss functions and edge embeddings), all-pair expressivity and algorithmic complexity w.r.t. node number $N$ and edge number $E$ (often $E\ll N^2$). The largest demonstration means the largest graph size used by the papers. ($m:$ number of sampled nodes, $s:$ number of augmentations, $h:$ number of hops.)\label{tab:comparison}}
 % \vspace{-5pt}
	\resizebox{\textwidth}{!}{
		\setlength{\tabcolsep}{3mm}{
			\begin{tabular}{c | cccc | c | cc | c }
				\toprule
				\multirow{2}{*}{\textbf{Model}}  &   \multicolumn{4}{c|}{\textbf{Model Components}} & \textbf{All-pair} & \multicolumn{2}{c|}{\textbf{Algorithmic Complexity}} & \textbf{Largest} \\
                & \textbf{Pos Enc} & \textbf{Multi-Layer Attn} & \textbf{Aug Loss} & \textbf{Edge Emb} & \textbf{Expressivity} & \textbf{Pre-processing}        & \textbf{Training} & \textbf{Demo.}                          \\
				\midrule
    GraphTransformer~\cite{graphtransformer-2020} & R & R & - & R  & Yes & $ O(N^3) $        & $ O(N^2) $ & 0.2K \\
    Graphormer~\cite{graphformer-neurips21}    & R & R & - & R & Yes & $ O(N^3)   $      & $ O(N^2)  $   & 0.3K      \\
				GraphTrans~\cite{graphtrans-neurips21}    & - & R & - & - & Yes & -                              & $ O(N^2)   $    & 0.3K                  \\
				% GraphiT~\cite{graphit-2021}  &  &       &$  O(N^3) $ & $ O(N^2 d) $                                                 & $ O(N^2+Nd) $   \\
                SAT~\cite{SAT-icml22}  & R & R & -  &  - & Yes &$  O(N^3) $ & $ O(N^2) $  & 0.2K                             \\
                EGT~\cite{EGT}   & R & R & R  &  R  & Yes & $ O(N^3) $        & $ O(N^2)  $  & 0.5K  \\
                GraphGPS~\cite{graphgps} & R & R & - & R & Yes & $O(N^3)$ & $O(N + E)$ & 1.0K \\
                \midrule
                Graph-Bert~\cite{graphbert-2020} & R & R & R & R & Yes & $ O(N^2)   $      & $ O(N^2) $ &  20K  \\
                Gophormer~\cite{gophormer}  & R & R & R & - & No & - & $O(Nsm^2)$ & 20K \\
                NodeFormer~\cite{nodeformer}  & R & R & R & - & Yes & - & $ O(N+E) $ & 2.0M  \\
                ANS-GT~\cite{ANS-GT} & R & R & - & - & No & - & $O(Nsm^2)$ & 20K \\NAGphormer~\cite{chen2023nagphormer} & R & R & - & R & No & $O(N^3)$ & $O(Nh^2)$ & 2.0M \\DIFFormer~\cite{wu2023difformer}  & - & R & - & - & Yes & - & $ O(N+E) $ & 1.6M  \\
				\textbf{SGFormer} (ours)  & - & - & - & - & Yes & - & $ O(N+E) $ & 0.1B \\
				\bottomrule
			\end{tabular}
	}
 }
 % \vspace{-5pt}
\end{table*}

%\textbf{Efficient Global Attention.} There are quite a few recent works on efficient Transformers that reduces the complexity of attention computation to $O(N)$ via, e.g., random feature map~\cite{performer-iclr21}, low-rank approximation~\cite{linformer} or local-global attention~\cite{longformer}. The first two approaches resort to approximation techniques for the global attention, while the third approach only attends on a small subset of tokens at each layer. In contrast with them, SGFormer guarantees the exact computation for all-pair attention without using any approximation and still achieves linear complexity. This is enabled by resorting to a new attention function \eqref{eqn-attn-ours}.

$\bullet$\; \textbf{Architectures.} Regarding model architectures, some existing models incorporate edge/positional embeddings (e.g., Laplacian decomposition features~\cite{graphtransformer-2020}, degree centrality~\cite{graphformer-neurips21}, Weisfeiler-Lehman labeling~\cite{graphbert-2020}) or utilize augmented training loss (e.g., edge regularization~\cite{EGT,nodeformer}) to capture graph information. However, the positional embeddings require an additional pre-processing procedure with a complexity of up to $O(N^3)$, which can be time- and memory-consuming for large graphs, while the augmented loss may complicate the optimization process. Moreover, existing models typically adopt a default design of stacking deep multi-head attention layers for competitive performance. In contrast, SGFormer does not require any of positional embeddings, augmented loss or pre-processing, and only uses a single-layer, single-head global attention, making it both efficient and lightweight.

$\bullet$\; \textbf{Expressivity.} There are some recently proposed graph Transformers for large graphs~\cite{gophormer,chen2023nagphormer,ANS-GT} that limit the attention computation to a subset of nodes, such as neighboring nodes or sampled nodes from the graph. This approach allows linear scaling w.r.t. graph sizes, but sacrifices the expressivity for accommodating all-pair interactions. In contrast, SGFormer maintains attention computation over all $N$ nodes in each layer while still achieving $O(N)$ complexity. Moreover, unlike NodeFormer~\cite{nodeformer} and GraphGPS~\cite{graphgps} which rely on random feature maps as approximation, SGFormer does not require any approximation or stochastic components and is more stable during training.

$\bullet$\; \textbf{Scalability.} In terms of algorithmic complexity, most existing graph Transformers have $O(N^2)$ complexity due to global all-pair attention, which is a critical computational bottleneck that hinders their scalability even for medium-sized graphs with thousands of nodes. While neighbor sampling can serve as a plausible remedy, it often sacrifices performance due to the significantly reduced receptive field~\cite{survey-graphtransformer}. SGFormer scales linearly w.r.t. $N$ and supports full-batch training on large graphs with up to 0.1M nodes. For further larger graphs, SGFormer is compatible with mini-batch training using large batch sizes, which allows the model to capture informative global information while having a negligible impact on performance. Notably, due to the linear complexity and simple architecture, SGFormer can scale to the web-scale graph \textsc{ogbn-papers100M} (with 0.1B nodes) when trained on a single GPU, two orders-of-magnitude larger than the largest demonstration among most graph Transformers.

\section{Empirical Evaluation}\label{sec-exp}
We apply SGFormer to real-world graph datasets whose predictive tasks can be modeled as node-level prediction. The latter is commonly used for effectiveness evaluation of learning graph representations and scalability to large graphs. We present the details of implementation and datasets in Sec.~\ref{sec-exp-detail}. Then in Sec.~\ref{sec-exp-comp-m}, we test SGFormer on medium-sized graphs (from 2K to 30K nodes) and compare it with an extensive set of expressive GNNs and Transformers. In Sec.~\ref{sec-exp-comp-l}, we scale SGFormer to large-sized graphs (from 0.1M to 0.1B nodes) where its superiority is demonstrated over scalable GNNs and Transformers. Later in Sec.~\ref{sec-exp-ratio}, we further compare the performance with different ratios of labeled data. In addition, we compare the model's time and space efficiency and scalability in Sec.~\ref{sec-exp-effi}. In Sec.~\ref{sec-exp-ablation}, we analyze the impact of several key components in our model. Sec.~\ref{sec-exp-dis} provides further discussions on how the single-layer model performs compared with the multi-layer counterpart.

\begin{table*}[tb!]
\centering
\caption{Mean and standard deviation of testing scores on medium-sized graph benchmarks. We annotate the node and edge number of each dataset and OOM indicates out-of-memory when training on a GPU with 24GB memory. We mark the model ranked in the \color{color1}{\textbf{first}}/\color{color2}{\textbf{second}}/\color{color3}{\textbf{third}} \color{black}place.}
% \vspace{-5pt}
\label{tab:main-res-m}
\resizebox{0.8\textwidth}{!}{
\begin{tabular}{l|c|c|c|c|c|c|c}
\toprule
\textbf{Dataset} & \textsc{cora} & \textsc{citeseer} & \textsc{pubmed} & \textsc{actor} & \textsc{squirrel} & \textsc{chameleon} & \textsc{deezer} \\
\midrule
\textbf{\# Nodes} & 2,708 & 3,327 & 19,717 & 7,600 & 2223 & 890 & 28,281 \\
\textbf{\# Edges} & 5,278 & 4,552 & 44,324 & 29,926 & 46,998 & 8,854 & 92,752 \\
% \textbf{Metric} & Accuracy & Accuracy & Accuracy & Accuracy & Accuracy & Accuracy & Accuracy \\
\midrule
GCN & 81.6 ± 0.4 & 71.6 ± 0.4 & 78.8 ± 0.6 & 30.1 ± 0.2 & 38.6 ± 1.8 & 41.3 ± 3.0 & 62.7 ± 0.7 \\
GAT & 83.0 ± 0.7 & 72.1 ± 1.1 & 79.0 ± 0.4 & 29.8 ± 0.6 & 35.6 ± 2.1 & 39.2 ± 3.1 & 61.7 ± 0.8 \\
SGC & 80.1 ± 0.2 & 71.9 ± 0.1 & 78.7 ± 0.1 & 27.0 ± 0.9 & 39.3 ± 2.3 & 39.0 ± 3.3 & 62.3 ± 0.4 \\
JKNet & 81.8 ± 0.5 & 70.7 ± 0.7 & 78.8 ± 0.7 & 30.8 ± 0.7 & 39.4 ± 1.6 & 39.4 ± 3.8 & 61.5 ± 0.4 \\
APPNP & \color{color3}\textbf{83.3 ± 0.5} & 71.8 ± 0.5 & \color{color3}\textbf{80.1 ± 0.2} & 31.3 ± 1.5 & 35.3 ± 1.9 & 38.4 ± 3.5 & 66.1 ± 0.6 \\
H2GCN & 82.5 ± 0.8 & 71.4 ± 0.7 & 79.4 ± 0.4 & 34.4 ± 1.7 & 35.1 ± 1.2 & 38.1 ± 4.0 & 66.2 ± 0.8 \\
SIGN & 82.1 ± 0.3 & 72.4 ± 0.8 & 79.5 ± 0.5 & 36.5 ± 1.0 & 40.7 ± 2.5 & 41.7 ± 2.2 & 66.3 ± 0.3 \\
CPGNN & 80.8 ± 0.4 & 71.6 ± 0.4 & 78.5 ± 0.7 & 34.5 ± 0.7 & 38.9 ± 1.2 & 40.8 ± 2.0 & 65.8 ± 0.3 \\
GloGNN & 81.9 ± 0.4 & 72.1 ± 0.6 & 78.9 ± 0.4 & 36.4 ± 1.6 & 35.7 ± 1.3 & 40.2 ± 3.9 & 65.8 ± 0.8 \\
\midrule
Graphormer$_\textsc{Small}$ & OOM & OOM & OOM & OOM & OOM & OOM & OOM \\
Graphormer$_\textsc{Smaller}$ & 75.8 ± 1.1 & 65.6 ± 0.6 & OOM & OOM & 40.9 ± 2.5 & 41.9 ± 2.8 & OOM \\
Graphormer$_\textsc{UltrasSmall}$ & 74.2 ± 0.9 & 63.6 ± 1.0 & OOM & 33.9 ± 1.4 & 39.9 ± 2.4 & 41.3 ± 2.8 & OOM \\
GraphTrans$_\textsc{Small}$ & 80.7 ± 0.9 & 69.5 ± 0.7 & OOM & 32.6 ± 0.7 & 41.0 ± 2.8 & \color{color2}\textbf{42.8 ± 3.3} & OOM \\
GraphTrans$_\textsc{UltrasSmall}$ & 81.7 ± 0.6 & 70.2 ± 0.8 & 77.4 ± 0.5 & 32.1 ± 0.8 & 40.6 ± 2.4 & 42.2 ± 2.9 & OOM \\
NodeFormer & 82.2 ± 0.9 & \color{color3}\textbf{72.5 ± 1.1} & 79.9 ± 1.0 & \color{color3}\textbf{36.9 ± 1.0} & 38.5 ± 1.5 & 34.7 ± 4.1 & 66.4 ± 0.7 \\
GraphGPS & 80.9 ± 1.1 & 68.6 ± 1.5 & 78.5 ± 0.7 & \color{color2}\textbf{37.1 ± 1.5} & \color{color3}\textbf{41.2 ± 2.1} &
42.5 ± 4.0 & \color{color3}\textbf{66.7 ± 0.3} \\ 
ANS-GT & 82.4 ± 0.9 & 70.7 ± 0.6 & 79.6 ± 1.0 & 35.8 ± 1.4 & 40.7 ± 1.4 & \color{color3}\textbf{42.6 ± 2.8} & 66.5 ± 0.7 \\
DIFFormer & \color{color1}\textbf{85.9 ± 0.4} & \color{color1}\textbf{73.5 ± 0.3} & \color{color1}\textbf{81.8 ± 0.3} & 36.5 ± 0.7 & \color{color2}\textbf{41.6 ± 2.5} & 42.5 ± 2.5 & \color{color2}\textbf{66.9 ± 0.7} \\
\textbf{SGFormer} & \color{color2}\textbf{84.5 ± 0.8} & \color{color2}\textbf{72.6 ± 0.2} & \color{color2}\textbf{80.3 ± 0.6} & \color{color1}\textbf{37.9 ± 1.1} & \color{color1}\textbf{41.8 ± 2.2} & \color{color1}\textbf{44.9 ± 3.9} & \color{color1}\textbf{67.1 ± 1.1}\\
\bottomrule
\end{tabular}
}
% \vspace{-5pt}
\end{table*}

\subsection{Experiment Details}\label{sec-exp-detail}

\textbf{Datasets.} We evaluate the model on 12 real-world datasets with diverse properties. Their sizes, as measured by the number of nodes in the graph, range from thousand-level to billion-level. We use 0.1M as the threshold and group these datasets into medium-sized datasets (with less than 0.1M nodes) and large-sized datasets (with more than 0.1M nodes). The medium-sized datasets include three citation networks \textsc{cora}, \textsc{citeseer} and \textsc{pubmed}~\cite{Sen08collectiveclassification}, where the graphs have high homophily ratios, and four heterophilic graphs \textsc{actor}~\cite{actor-kdd09}, \textsc{squirrel}, \textsc{chameleon}~\cite{rozemberczki2021multiscale} and \textsc{deezer-europe}~\cite{deezer-2020}, where neighboring nodes tend to have distinct labels. These graphs have 2K-30K nodes and the detailed statistics are reported in Table~\ref{tab:main-res-m}. The large-sized datasets include the citation networks \textsc{ogbn-arxiv} and \textsc{pgbn-papers100M}, the protein interaction network \textsc{ogbn-proteins}~\cite{ogb-nips20}, the item co-occurrence network \textsc{Amazon2M}~\cite{amazoncopurchase-kdd15}, and the social network \textsc{pokec}~\cite{leskovec2016snap}. In particular, \textsc{Amazon2M} entails long-range dependency and \textsc{pokec} is a heterophilic graph. The detailed statistics are presented in Table~\ref{tab:main-res-l} and the largest dataset \textsc{ogbn-papers100M} contains more than 0.1B nodes.

\textbf{Implementation.} The input layer $f_I$ is instantiated as a fully-connected layer with ReLU activation. The output layer $f_O$ is instantiated as a fully-connected layer (with Softmax for classification). The GN module is basically implemented as a GCN~\cite{GCN-vallina} with shallow (e.g., 1-3) propagation layers. We use different training schemes for graph datasets of different scales. For medium-sized graphs, we use full-graph training: the whole graph dataset is fed into the model during training and inference. For large-sized graphs, we adopt the mini-batch training as introduced in Sec.~\ref{sec-model-design}. In specific, we set the batch size as 10K, 0.1M, 0.1M and 0.4M for \textsc{ogbn-proteins}, \textsc{Amazon2M}, \textsc{pokec} and \textsc{ogbn-papers100M}, respectively. Then for inference on these large-sized graphs, following the pipeline used by \cite{ogb-nips20}, we feed the whole graph into the model using CPU, which allows computing the all-pair attention among all the nodes in the dataset. Particularly, for the gigantic graph \textsc{ogbn-papers100M} that cannot be fed as whole into common CPU with moderate memory, we adopt the mini-batch partition strategy used in training to reduce the overhead. For hyper-parameter settings, we use the model performance on the validation set as the reference. Unless otherwise stated, the hyper-parameters are selected using grid search with the searching space: learning rate within $\{0.001, 0.005, 0.01, 0.05, 0.1\}$, weight decay within $\{1e-5, 1e-4, 5e-4, 1e-3, 1e-2\}$, hidden size within $\{32, 64, 128, 256\}$, dropout ratio within $\{0, 0.2, 0.3, 0.5\}$, weight $\alpha$ within $ \{0.5, 0.8\} $.

\textbf{Evaluation Protocol.} We follow the common practice and set a fixed number of training epochs: 300 for medium-sized graphs, 1000 for large-sized graphs, and 50 for the extremely large graph \textsc{ogbn-papers100M}. We use ROC-AUC as the evaluation metric for \textsc{ogbn-proteins} and Accuracy for other datasets. The testing score achieved by the model that reports the highest score on the validation set is used for evaluation. We run each experiment with five independent trials using different initializations, and report the mean and variance of the metrics for comparison.

\subsection{Comparative Results on Medium-sized Graphs}\label{sec-exp-comp-m}

\textbf{Setup.} We first evaluate the model on medium-sized datasets. For citation networks \textsc{cora}, \textsc{citeseer} and \textsc{pubmed}, we follow the commonly used benchmark setting, i.e., semi-supervised data splits adopted by \cite{GCN-vallina}. For \textsc{actor} and \textsc{deezer-europe}, we use the random splits of the benchmark setting introduced by \cite{newbench}. For \textsc{squirrel} and \textsc{chameleon}, we use the splits proposed by a recent evaluation paper~\cite{platonov2023a} that filters the overlapped nodes in the original datasets. 

\textbf{Competitors.} Given the moderate sizes of graphs where most of existing models can scale smoothly, we compare with multiple sets of competitors from various aspects. Basically, we adopt standard GNNs including GCN~\cite{GCN-vallina}, GAT~\cite{GAT} and SGC~\cite{SGC-icml19} as baselines. Besides, we compare with advanced GNN models, including JKNet~\cite{jknet-icml18}, APPNP~\cite{appnp}, SIGN~\cite{sign-2020}, H2GCN~\cite{h2gcn-neurips20}, CPGNN~\cite{zhu2021graph} and GloGNN~\cite{glognn}. In terms of Transformers, we mainly compare with the state-of-the-art scalable graph Transformers NodeFormer~\cite{nodeformer}, GraphGPS~\cite{graphgps}, ANS-GT~\cite{ANS-GT} and DIFFormer~\cite{wu2023difformer}. Furthermore, we adapt two powerful Transformers tailored for graph-level tasks, i.e., Graphormer~\cite{graphformer-neurips21} and GraphTrans~\cite{graphtrans-neurips21}, for comparison.
In particular, since the original implementations of these models are of large sizes and are difficult to scale on all node-level prediction datasets considered in this paper, we adopt their smaller versions for experiments. We use Graphormer$ _\textsc{small} $ (6 layers and 32 heads), Graphormer$ _\textsc{smaller} $ (3 layers and 8 heads) and Graphormer$ _\textsc{UltraSmall} $ (2 layers and 1 head). As for GraphTrans, we use GraphTrans$ _\textsc{small} $(3 layers and 4 heads) and GraphTrans$ _\textsc{UltraSmall} $ (2 layers and 1 head).

\textbf{Results.} Table~\ref{tab:main-res-m} reports the results of all the models. We found that SGFormer significantly outperforms three standard GNNs (GCN, GAT and SGC) by a large margin, with up to $25.9\%$ impv. over GCN on \textsc{actor}, which suggests that our single-layer global attention model is indeed effective despite its simplicity. Moreover, we observe that the relative improvements of SGFormer over three standard GNNs are overall more significant on heterophilic graphs \textsc{actor}, \textsc{squirrel}, \textsc{chameleon} and \textsc{deezer}. The possible reason is that in such cases the global attention could help to filter out spurious edges from neighboring nodes of different classes and accommodate dis-connected yet informative nodes in the graph. Compared to other advanced GNNs and graph Transformers (NodeFormer, ANS-GT and GraphGPS), the performance of SGFormer is highly competitive and even superior with significant gains over the runner-ups in most cases. These results serve as concrete evidence for verifying the efficacy of SGFormer as a powerful learner for node-level prediction. We also found that both Graphormer and GraphTrans suffer from serious over-fitting, due to their relatively complex architectures and limited ratios of labeled nodes. In contrast, the simple and lightweight architecture of SGFormer leads to its better generalization ability given the limited supervision in these datasets.

\begin{table*}[tb!]
\centering
\caption{Testing results on large-sized graph benchmarks. OOT indicates that the training cannot be finished within an acceptable time budget.}
\vspace{-5pt}
\label{tab:main-res-l}
\resizebox{0.75\linewidth}{!}{
\begin{tabular}{l|c|c|c|c|c}
\toprule
\textbf{Dataset} & \textsc{ogbn-proteins} & \textsc{Amazon2M} & \textsc{pokec} & \textsc{ogbn-arxiv} & \textsc{ogbn-papers100M} \\
\midrule
\textbf{\# Nodes} & 132,534 & 2,449,029 & 1,632,803 & 169,343 & 111,059,956 \\
\textbf{\# Edges} & 39,561,252 & 61,859,140 & 30,622,564 & 1,166,243 & 1,615,685,872 \\
% \textbf{Metric} & ROC-AUC & Accuracy & Accuracy & Accuracy & Accuracy \\
\midrule
MLP & 72.04 ± 0.48 & 63.46 ± 0.10 & 60.15 ± 0.03 & 55.50 ± 0.23 & 47.24 ± 0.31 \\
GCN & 72.51 ± 0.35 & 83.90 ± 0.10 & 62.31 ± 1.13 & \color{color2}\textbf{71.74 ± 0.29} & OOM \\
SGC & 70.31 ± 0.23 & 81.21 ± 0.12 & 52.03 ± 0.84 & 67.79 ± 0.27 & 63.29 ± 0.19 \\
GCN-NSampler & 73.51 ± 1.31 & 83.84 ± 0.42 & 63.75 ± 0.77 & 68.50 ± 0.23 & 62.04 ± 0.27 \\
GAT-NSampler & 74.63 ± 1.24 & 85.17 ± 0.32 & 62.32 ± 0.65 & 67.63 ± 0.23 & \color{color3}\textbf{63.47 ± 0.39} \\
SIGN & 71.24 ± 0.46 & 80.98 ± 0.31 & 68.01 ± 0.25 & \color{color3}\textbf{70.28 ± 0.25} & \color{color2}\textbf{65.11 ± 0.14} \\
NodeFormer & \color{color3}\textbf{77.45 ± 1.15} & \color{color2}\textbf{87.85 ± 0.24} & \color{color2}\textbf{70.32 ± 0.45} & 59.90 ± 0.42 & OOT  \\
DIFFormer & \color{color2}\textbf{79.49 ± 0.44} & \color{color3}\textbf{85.21 ± 0.62}
 & \color{color3}\textbf{69.24 ± 0.76} & 68.52 ± 0.49 & OOT \\
\textbf{SGFormer} & \color{color1}\textbf{79.53 ± 0.38} & \color{color1}\textbf{89.09 ± 0.10} & \color{color1}\textbf{73.76 ± 0.24} & \color{color1}\textbf{72.63 ± 0.13} & \color{color1}\textbf{66.01 ± 0.37}\\
\bottomrule
\end{tabular}
}
% \vspace{-5pt}
\end{table*}

\begin{table*}[tb!]
\centering
\caption{Efficiency comparison of SGFormer and graph Transformer competitors w.r.t. training time per epoch, inference time and GPU memory costs on a Tesla T4. We use the small model versions of Graphormer and GraphTrans. The missing results are caused by the out-of-memory issue. %For \textsc{cora} and \textsc{pubmed}, we report the averaged time for computing one full-batch; for \textsc{Amazon2M}, we report the averaged time for computing a mini-batch of 0.1M nodes.
}
\vspace{-5pt}
\label{tbl:efficiency}
% Please add the following required packages to your document preamble:
% \usepackage{multirow}
\resizebox{0.8\linewidth}{!}{
\begin{tabular}{l|ccc|ccc|ccc}
\toprule
\multirow{2}{*}{\textbf{Model}} & \multicolumn{3}{c|}{\textsc{cora}} & \multicolumn{3}{c|}{\textsc{pubmed}} & \multicolumn{3}{c}{\textsc{Amazon2M}} \\
 & \textbf{Tr (ms)} & \textbf{Inf (ms)} & \textbf{Mem (GB)} & \textbf{Tr (ms)} & \textbf{Inf (ms)} & \textbf{Mem (GB)} & \textbf{Tr (ms)} & \textbf{Inf (ms)} & \textbf{Mem (GB)} \\
 \midrule
 Graphormer &563.5 & 537.1 &5.0& - & - & -& - & - & - \\
 GraphTrans &160.4&40.2&3.8& - & - & -& - & - & -\\
    NodeFormer & 68.5 & 30.2 & 1.2 & 321.4 & 135.5 & 2.9 & 5369.5 & 1410.0 & 4.6 \\
    GraphGPS & 60.8 & 26.1 & 1.2 & 423.1 & 217.7 & 1.6 & - & - & - \\
    ANS-GT & 570.1 & 539.2 & 1.0 & 511.9 & 461.0 & 2.1 & - & - & -\\
    DIFFormer & 49.7 & 9.6 & 1.2 & 85.3 & 30.8 & 2.4 & 3683.8 & 523.1 & 4.5 \\
    \textbf{SGFormer} & 15.0 & 3.8 & 0.9 & 15.4 & 4.4 & 1.0 & 2481.4 & 382.5 & 2.7 \\
\bottomrule
\end{tabular}
}
% \vspace{-5pt}
\end{table*}

\begin{figure}[tb!]
  \centering
  \includegraphics[width=\linewidth]{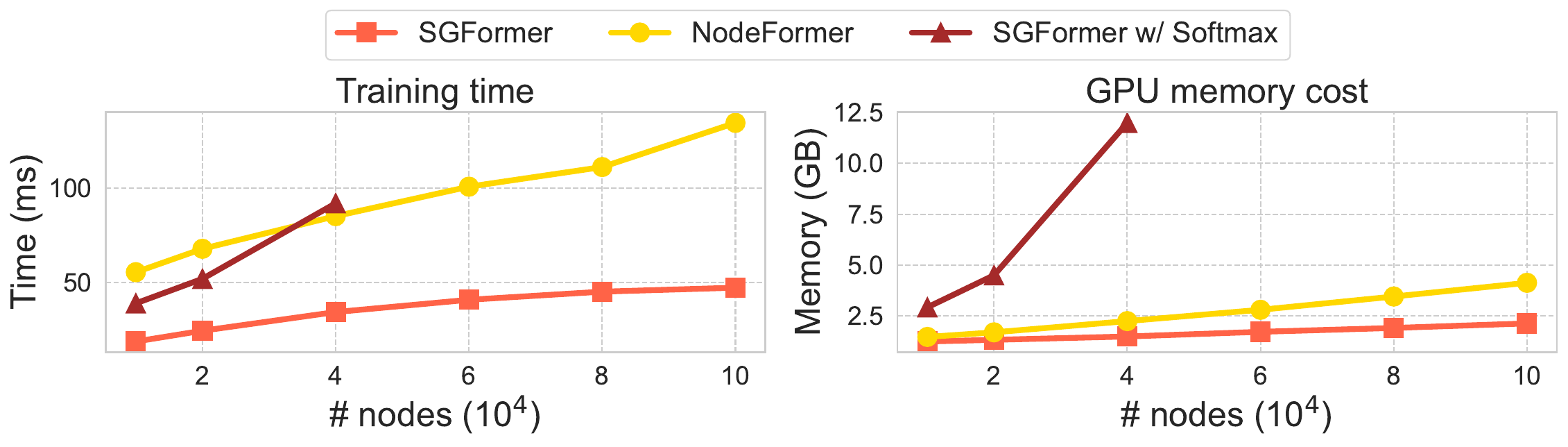}
  % \vspace{-5pt}
  \caption{Scalability test of training time per epoch and GPU memory cost w.r.t. graph sizes (a.k.a. node numbers). NodeFormer reports OOM when \# nodes reaches more than 30K.}
  \label{fig-scalability}
  % \vspace{-5pt}
\end{figure}

\subsection{Comparative Results on Large-sized Graphs}\label{sec-exp-comp-l}

\textbf{Setup.} We further evaluate the model on large-sized graph datasets where the numbers of nodes range from millions to billions. For three OGB datasets \textsc{ogbn-arxiv}, \textsc{ogbn-proteins} and \textsc{ogbn-papers100M}, we use the public splits provided by \cite{ogb-nips20}. Furthermore, we follow the splits used by the recent work \cite{nodeformer} for \textsc{Amazon2M} and adopt random splits with the ratio 1:1:8 for \textsc{pokec}.

\textbf{Competitors.} Due to the large graph sizes, most of the expressive GNNs and Transformers compared in Sec.~\ref{sec-exp-comp-m} are hard to scale within the acceptable budget of time and memory costs. Therefore, we compare with MLP, GCN and two scalable GNNs, i.e., SGC and SIGN. We also compare with GNNs using the neighbor sampling technique~\cite{graphsaint}: GCN-NSampler and GAT-NSampler. Our main competitors are NodeFormer~\cite{nodeformer} and DIFFormer~\cite{wu2023difformer}, the recently proposed scalable graph Transformers with all-pair attention. 

\textbf{Results.} Table~\ref{tab:main-res-l} presents the experimental results. We found that SGFormer yields consistently superior results across five datasets, with significant performance improvements over GNN competitors. This suggests the effectiveness of the global attention that can learn implicit inter-dependencies among a large number of nodes beyond input structures. Furthermore, SGFormer outperforms NodeFormer by a clear margin across all the cases, which demonstrates the superiority of SGFormer that uses simpler architecture and achieves better performance on large graphs. For the largest dataset \textsc{ogbn-papers100M} where prior Transformer models fail to demonstrate, SGFormer scales smoothly with decent efficiency and yields highly competitive results. Specifically, SGFormer reaches the testing accuracy of 66.0 with consumption of about 3.5 hours and 23.0 GB memory on a single GPU for training. This result provides strong evidence that shows the promising power of SGFormer on extremely large graphs, producing superior performance with limited computation budget.

\begin{table*}[tb!]
\centering
\caption{Testing results with different training ratios on \textsc{cora}.}
\vspace{-5pt}
\label{tab:ratio-res-cora}
\resizebox{\textwidth}{!}{
\begin{tabular}{l|c|c|c|c|c|c|c|c|c|c}
\toprule
\textbf{Training Ratio} & \multicolumn{1}{c|}{\textbf{50\%}} & \multicolumn{1}{c|}{\textbf{40\%}} & \multicolumn{1}{c|}{\textbf{30\%}} & \multicolumn{1}{c|}{\textbf{20\%}} & \multicolumn{1}{c|}{\textbf{10\%}} & \multicolumn{1}{c|}{\textbf{5\%}} & \multicolumn{1}{c|}{\textbf{4\%}}& \multicolumn{1}{c|}{\textbf{3\%}} & \multicolumn{1}{c|}{\textbf{2\%}} & \multicolumn{1}{c}{\textbf{1\%}} \\
\midrule
GCN & 87.71 ± 0.56 & 87.26 ± 0.20 & 86.33 ± 0.37 & 85.37 ± 0.42 & 83.06 ± 0.74 & 
\color{color3}\textbf{79.94 ± 1.01} & \color{color3}\textbf{79.36 ± 0.33} & 76.16 ± 0.49 & 67.50 ± 0.23 & 62.50 ± 1.22\\ 
GAT & 87.86 ± 0.32 & \color{color3}\textbf{87.41 ± 0.26} & 86.32 ± 0.33 & 85.53 ± 0.58 & 83.24 ± 0.99 & 
79.84 ± 0.48 & 79.24 ± 0.61 & 75.63 ± 0.51 & 67.16 ± 0.51 & 63.30 ± 0.76\\
SGC & 86.50 ± 0.13 & 84.56 ± 0.52 & 83.04 ± 0.57 & 80.43 ± 0.14 & 77.51 ± 0.21 & 
72.49 ± 0.24 & 71.00 ± 0.80 & 67.08 ± 0.27 & 60.39 ± 0.31 & 52.87 ± 0.22\\
JKNet & 87.42 ± 0.69 & 86.18 ± 0.58 & 85.53 ± 0.45 & 82.85 ± 0.49 & 79.92 ± 0.24 & 
75.33 ± 0.76 & 74.77 ± 0.87 & 71.55 ± 1.09 & 62.69 ± 0.44 & 56.09 ± 1.50\\
APPNP & \color{color3}\textbf{88.51 ± 0.19} & 88.21 ± 0.46 & \color{color2}\textbf{87.40 ± 0.11} & \color{color3}\textbf{86.08 ± 0.25} & \color{color2}\textbf{84.55 ± 0.36} & 
\color{color2}\textbf{81.55 ± 0.33} & \color{color2}\textbf{80.53 ± 0.34} & \color{color2}\textbf{77.99 ± 0.26} & \color{color3}\textbf{68.65 ± 0.15} & \color{color3}\textbf{64.55 ± 0.93}\\
H2GCN & \color{color2}\textbf{88.63 ± 0.28} & \color{color2}\textbf{87.64 ± 0.58} & \color{color3}\textbf{87.28 ± 0.23} & \color{color2}\textbf{86.19 ± 0.27} & \color{color3}\textbf{83.35 ± 0.49} & 
79.47 ± 0.79 & 79.29 ± 0.36 & \color{color3}\textbf{77.16 ± 0.59} & 67.01 ± 1.15 & 64.04 ± 0.67\\
SIGN & 85.29 ± 0.08 & 84.03 ± 0.28 & 82.85 ± 0.29 & 80.81 ± 1.24 & 76.73 ± 1.01 & 
71.38 ± 0.49 & 67.99 ± 1.10 & 64.35 ± 0.98 & 55.95 ± 0.39 & 50.20 ± 0.68\\
\midrule
Graphormer & 75.07 ± 0.80 & 74.24 ± 0.95 & 71.60 ± 0.97 & 68.87 ± 0.23 & 62.23 ± 0.99 & 
54.03 ± 0.73 & 52.77 ± 0.43 & 49.15 ± 1.36 & 43.15 ± 0.64 & 39.06 ± 1.55\\
GraphTrans & 87.62 ± 0.40 & 87.22 ± 0.62 & 86.25 ± 0.48 & 84.34 ± 0.81 & 81.64 ± 0.75 & 
79.00 ± 0.34 & 78.48 ± 1.23 & 75.12 ± 1.02 & 67.29 ± 0.80 & 63.68 ± 0.98\\
NodeFormer & 87.36 ± 0.81 & 86.41 ± 0.82 & 84.89 ± 1.15 & 82.81 ± 1.10 & 77.71 ± 1.91 & 
67.62 ± 0.63 & 67.84 ± 0.63 & 65.85 ± 0.92 & 58.12 ± 0.20 & 53.47 ± 1.90\\
GraphGPS & 86.97 ± 0.59 & 85.55 ± 0.33 & 85.56 ± 0.55 & 81.81 ± 0.53 & 79.25 ± 0.94 & 
75.76 ± 0.98 & 75.46 ± 1.55 & 71.58 ± 2.16 & 63.36 ± 1.92 & 55.06 ± 1.65\\
ANS-GT & 83.79 ± 0.27 & 83.53 ± 0.64 & 83.31 ± 0.34 & 81.82 ± 0.81 & 80.64 ± 0.11 & 79.87 ± 0.57 & 77.96 ± 0.63 & 74.74 ± 0.47 & \color{color2}\textbf{69.63 ± 0.24} & \color{color2}\textbf{64.72 ± 0.84} \\
DIFFormer & 88.24 ± 0.31 & 86.73 ± 0.40 & 86.58 ± 0.68 & 83.81 ± 0.35 & 80.79 ± 0.72 & 
78.34 ± 0.82 & 78.24 ± 1.38 & 77.01 ± 0.96 & 66.44 ± 1.47 & 63.19 ± 0.32\\
\textbf{SGFormer} & \color{color1}\textbf{89.22 ± 0.18} & \color{color1}\textbf{88.85 ± 0.04} & \color{color1}\textbf{87.85 ± 0.14} & \color{color1}\textbf{87.81 ± 0.04} & \color{color1}\textbf{86.11 ± 0.23} & 
\color{color1}\textbf{82.05 ± 0.43} & \color{color1}\textbf{81.40 ± 0.41} & \color{color1}\textbf{77.39 ± 0.94} & \color{color1}\textbf{69.84 ± 0.33} & \color{color1}\textbf{65.14 ± 0.66}\\
\toprule
\end{tabular}
}
% \vspace{-5pt}
\end{table*}

\begin{table*}[tb!]
\centering
\caption{Testing results with different training ratios on \textsc{pubmed}.}
\vspace{-5pt}
\label{tab:ratio-res-pubmed}
\resizebox{\textwidth}{!}{
\begin{tabular}{l|c|c|c|c|c|c|c|c|c|c}
\toprule
\textbf{Training Ratio} & \multicolumn{1}{c|}{\textbf{50\%}} & \multicolumn{1}{c|}{\textbf{40\%}} & \multicolumn{1}{c|}{\textbf{30\%}} & \multicolumn{1}{c|}{\textbf{20\%}} & \multicolumn{1}{c|}{\textbf{10\%}} & \multicolumn{1}{c|}{\textbf{5\%}} & \multicolumn{1}{c|}{\textbf{4\%}}& \multicolumn{1}{c|}{\textbf{3\%}} & \multicolumn{1}{c|}{\textbf{2\%}} & \multicolumn{1}{c}{\textbf{1\%}} \\
\midrule
GCN & 86.74 ± 0.13 & 86.67 ± 0.07 & 84.59 ± 0.16 & 86.63 ± 0.07 & 85.77 ± 0.06 & 85.63 ± 0.04 & 83.59 ± 0.09 & 84.63 ± 0.03 & 83.21 ± 0.07 & 82.21 ± 0.13 \\
GAT & 88.26 ± 0.09 & 88.10 ± 0.07 & 86.36 ± 0.19 & 86.32 ± 0.12 & 85.63 ± 0.24 & 84.81 ± 0.20 & 83.50 ± 0.21 & 84.22 ± 0.24 & 82.88 ± 0.09 & 82.52 ± 0.29 \\
SGC & 81.99 ± 0.05 & 81.13 ± 0.10 & 80.96 ± 0.10 & 81.11 ± 0.08 & 80.56 ± 0.08 & 79.38 ± 0.06 & 79.83 ± 0.02 & 78.05 ± 0.04 & 79.52 ± 0.05 & 76.29 ± 0.03 \\
JKNet & 89.10 ± 0.31 & 88.39 ± 0.26 & 87.94 ± 0.20 & 87.74 ± 0.29 & 85.97 ± 0.31 & 84.22 ± 0.58 & 83.71 ± 0.50 & 83.03 ± 0.14 & 82.51 ± 0.37 & 81.19 ± 0.38 \\
APPNP & 87.55 ± 0.17 & 87.15 ± 0.08 & 86.73 ± 0.12 & 86.99 ± 0.11 & \color{color3}\textbf{86.26 ± 0.07} & \color{color3}\textbf{85.77 ± 0.08} & \color{color2}\textbf{85.40 ± 0.18} & \color{color2}\textbf{84.92 ± 0.09} & \color{color2}\textbf{84.22 ± 0.10} & \color{color2}\textbf{82.80 ± 0.23} \\
H2GCN & 89.49 ± 0.07 & 88.73 ± 0.04 & 84.28 ± 0.26 & 84.38 ± 0.10 & 83.98 ± 0.18 & 83.73 ± 0.18 & 83.86 ± 0.05 & 83.16 ± 0.09 & 83.52 ± 0.17 & \color{color3}\textbf{82.64 ± 0.07} \\
SIGN & 89.51 ± 0.04 & 88.88 ± 0.09 & 88.42 ± 0.06 & 87.73 ± 0.05 & 85.79 ± 0.14 & 84.13 ± 0.43 & 83.78 ± 0.20 & 82.77 ± 0.11 & 81.49 ± 0.04 & 78.75 ± 1.08 \\
\midrule
Graphormer & OOM & OOM & OOM & OOM & OOM & OOM & OOM & OOM & OOM & OOM \\
GraphTrans & 88.77 ± 0.12 & 88.30 ± 0.14 & 87.78 ± 0.26 & 87.40 ± 0.18 & 85.76 ± 0.16 & 83.63 ± 0.10 & 83.35 ± 0.06 & 83.19 ± 0.27 & 83.06 ± 0.06 & 82.08 ± 0.11 \\
NodeFormer & 88.77 ± 0.21 & 88.67 ± 0.29 & 88.28 ± 0.29 & \color{color3}\textbf{88.14 ± 0.13} & 86.01 ± 0.58 & \color{color2}\textbf{85.83 ± 0.22} & \color{color3}\textbf{85.33 ± 0.08} & \color{color3}\textbf{84.70 ± 0.28} & \color{color3}\textbf{84.11 ± 0.12} & 81.02 ± 1.97 \\
GraphGPS & \color{color3}\textbf{89.97 ± 0.15} & \color{color3}\textbf{89.43 ± 0.21} & \color{color2}\textbf{89.04 ± 0.18} & \color{color2}\textbf{88.24 ± 0.18} & \color{color2}\textbf{86.46 ± 0.16} & 85.14 ± 0.24 & 84.70 ± 0.16 & 83.92 ± 0.16 & 83.28 ± 0.25 & 80.53 ± 0.47 \\
ANS-GT & 86.73 ± 0.40 & 86.54 ± 0.37 & 86.14 ± 0.12 & 85.53 ± 0.19 & 83.67 ± 0.53 & 82.58 ± 0.42 & 82.33 ± 0.45 & 82.15 ± 0.64 & 81.85 ± 0.39 & 80.81 ± 0.74 \\
DIFFormer & \color{color2}\textbf{90.29 ± 0.10} & \color{color2}\textbf{89.66 ± 0.21} & \color{color3}\textbf{88.96 ± 0.16} & 87.73 ± 0.12 & 86.22 ± 0.25 & 85.50 ± 0.12 & 85.18 ± 0.30 & 84.73 ± 0.25 & 83.90 ± 0.43 & 82.55 ± 0.32 \\
\textbf{SGFormer} & \color{color1}\textbf{90.52 ± 0.19} & \color{color1}\textbf{89.90 ± 0.02} & \color{color1}\textbf{89.33 ± 0.09} & \color{color1}\textbf{88.53 ± 0.11} & \color{color1}\textbf{86.89 ± 0.08} & \color{color1}\textbf{86.24 ± 0.08} & \color{color1}\textbf{85.83 ± 0.11} & \color{color1}\textbf{85.01 ± 0.31} & \color{color1}\textbf{84.29 ± 0.13} & \color{color1}\textbf{82.93 ± 0.12} \\
\bottomrule
\end{tabular}
}
% \vspace{-5pt}
\end{table*}

\begin{table*}[tb!]
\centering
\caption{Testing results with different training ratios on \textsc{chameleon}.}
\vspace{-5pt}
\label{tab:ratio-res-chameleon}
\resizebox{\textwidth}{!}{
\begin{tabular}{l|c|c|c|c|c|c|c|c|c|c}
\toprule
\textbf{Training Ratio} & \multicolumn{1}{c|}{\textbf{50\%}} & \multicolumn{1}{c|}{\textbf{40\%}} & \multicolumn{1}{c|}{\textbf{30\%}} & \multicolumn{1}{c|}{\textbf{20\%}} & \multicolumn{1}{c|}{\textbf{10\%}} & \multicolumn{1}{c|}{\textbf{5\%}} & \multicolumn{1}{c|}{\textbf{4\%}}& \multicolumn{1}{c|}{\textbf{3\%}} & \multicolumn{1}{c|}{\textbf{2\%}} & \multicolumn{1}{c}{\textbf{1\%}} \\
\midrule
GCN & \color{color2}\textbf{42.87 ± 3.36} & \color{color2}\textbf{42.15 ± 2.98} & \color{color2}\textbf{41.75 ± 1.80} & 37.27 ± 3.09 & 34.42 ± 4.03 & \color{color2}\textbf{36.52 ± 1.88} & \color{color2}\textbf{36.13 ± 2.88} & \color{color2}\textbf{33.82 ± 5.01} & 31.52 ± 5.79 & \color{color2}\textbf{31.00 ± 5.97} \\
GAT & 38.79 ± 2.74 & 38.94 ± 2.26 & 39.05 ± 2.38 & 36.10 ± 1.81 & 36.13 ± 1.77 & 34.82 ± 1.32 & 30.76 ± 6.27 & \color{color3}\textbf{33.75 ± 4.70} & \color{color3}\textbf{32.27 ± 3.85} & 30.29 ± 4.59 \\
SGC & 35.16 ± 3.24 & 35.48 ± 3.07 & 33.72 ± 4.62 & 32.78 ± 5.38 & 32.69 ± 5.37 & 29.82 ± 6.63 & 27.41 ± 6.40 & 27.74 ± 6.72 & 26.54 ± 5.33 & 26.61 ± 5.71 \\
JKNet & 40.36 ± 2.47 & 39.62 ± 1.81 & 39.53 ± 1.80 & \color{color2}\textbf{38.84 ± 2.41} & \color{color2}\textbf{38.03 ± 2.36} & 34.81 ± 2.00 & 33.06 ± 4.24 & 31.62 ± 3.95 & 31.95 ± 5.33 & \color{color3}\textbf{30.76 ± 5.06} \\
APPNP & 37.13 ± 1.86 & 37.50 ± 2.46 & 37.81 ± 1.85 & 36.33 ± 2.02 & 35.68 ± 2.28 & 34.57 ± 1.23 & 34.36 ± 2.32 & 33.46 ± 4.34 & \color{color2}\textbf{32.70 ± 4.36} & 30.38 ± 5.96 \\
H2GCN & 35.70 ± 3.57 & 36.28 ± 2.51 & 35.51 ± 1.56 & 35.00 ± 1.43 & 33.94 ± 1.57 & 32.87 ± 1.43 & 32.97 ± 2.78 & 31.46 ± 4.30 & 30.28 ± 4.37 & 30.23 ± 4.40 \\
SIGN & 34.17 ± 3.18 & 34.01 ± 3.01 & 33.77 ± 1.97 & 33.59 ± 2.62 & 32.42 ± 2.54 & 30.13 ± 3.48 & 29.04 ± 3.46 & 28.02 ± 4.64 & 28.42 ± 4.23 & 28.05 ± 4.04 \\
\midrule
Graphormer & 27.09 ± 2.47 & 27.24 ± 1.77 & 27.18 ± 1.76 & 26.53 ± 2.91 & 25.85 ± 2.53 & 25.29 ± 3.36 & 23.87 ± 3.06 & 24.38 ± 3.08 & 24.47 ± 3.05 & 25.02 ± 2.71 \\
Graphtrans & 27.09 ± 2.47 & 27.18 ± 1.85 & 26.93 ± 1.99 & 26.69 ± 2.83 & 25.56 ± 2.52 & 25.29 ± 3.36 & 24.55 ± 3.19 & 24.28 ± 3.14 & 24.47 ± 3.05 & 25.02 ± 2.71 \\
NodeFormer & 39.64 ± 3.04 & 38.75 ± 2.84 & 37.58 ± 2.44 & 37.84 ± 2.81 & 35.58 ± 2.10 & 30.77 ± 4.02 & 31.55 ± 5.28 & 27.87 ± 5.37 & 27.88 ± 5.73 & 28.14 ± 4.96 \\
GraphGPS & \color{color3}\textbf{41.57 ± 3.01} & \color{color3}\textbf{41.63 ± 2.82} & \color{color3}\textbf{40.45 ± 1.28} & \color{color3}\textbf{38.47 ± 2.12} & \color{color3}\textbf{37.32 ± 2.03} & \color{color3}\textbf{35.02 ± 5.39} & \color{color3}\textbf{34.80 ± 2.73} & 33.40 ± 5.26 & 32.23 ± 4.72 & 30.18 ± 4.57 \\
ANS-GT & 38.49 ± 1.78 & 37.23 ± 1.92 & 40.09 ± 1.92 & 36.94 ± 2.11 & 35.59 ± 1.57 & 30.94 ± 2.02 & 28.70  ± 1.98 & 28.39 ± 3.42 & 29.28 ± 1.54 & 28.11 ± 1.10 \\
DIFFormer & 39.10 ± 2.97 & 39.04 ± 3.13 & 39.45 ± 1.70 & 38.41 ± 2.79 & 36.01 ± 2.46 & 33.24 ± 1.55 & 34.50 ± 1.35 & 32.77 ± 3.19 & 31.86 ± 3.72 & 30.36 ± 5.17 \\
\textbf{SGFormer} & \color{color1}\textbf{47.98 ± 2.88} & \color{color1}\textbf{43.97 ± 3.07} & \color{color1}\textbf{43.19 ± 2.02} & \color{color1}\textbf{40.94 ± 1.88} & \color{color1}\textbf{39.60 ± 1.49} & \color{color1}\textbf{38.25 ± 2.03} & \color{color1}\textbf{37.82 ± 3.96} & \color{color1}\textbf{35.53 ± 2.23} & \color{color1}\textbf{32.72 ± 5.23} & \color{color1}\textbf{31.62 ± 2.94} \\
\toprule
\end{tabular}
}
% \vspace{-5pt}
\end{table*}

\subsection{Efficiency and Scalability}\label{sec-exp-effi}

We next provide more quantitative comparisons w.r.t. the efficiency and scalability of SGFormer with the graph Transformer competitors in datasets of different scales.

\textbf{Efficiency.} Table~\ref{tbl:efficiency} reports the training time per epoch, inference time and GPU memory costs on \textsc{cora}, \textsc{pubmed} and \textsc{Amazon2M}. Since the common practice for model training in these datasets is to use a fixed number of training epochs, we report the training time per epoch here for comparing the training efficiency. We found that notably, SGFormer is orders-of-magnitude faster than other competitors. Compared to Graphormer and GraphTrans that require quadratic complexity for global attention and are difficult for scaling to graphs with thousands of nodes, SGFormer significantly reduces the memory costs due to the simple global attention of $O(N)$ complexity. In terms of time costs, it yields 38x/141x training/inference speedup over Graphormer on \textsc{cora}, and 20x/30x speedup over NodeFormer on \textsc{pubmed}. This is mainly because Graphormer requires the compututation of quadratic global attention that is time-consuming, and NodeFormer additionally employs the augmented loss and Gumbel tricks. For GPU costs, the memory usage of SGFormer, NodeFormer and ANS-GT on two medium-sized graphs is comparable, while on the large graph \textsc{Amazon2M}, SGFormer consumes much less memory than NodeFormer. On the large graph \textsc{Amazon2M}, where NodeFormer and SGFormer both leverage mini-batch training, SGFormer is 2x and 4x faster than NodeFormer regarding training and inference, respectively.

% \textbf{Memory Costs.} Due to $O(N^2)$ complexity, Graphormer and GraphTrans suffer out-of-memory issue on quite a few node classification datasets in Table~\ref{tab:main_res}. Even using the small or ultra-small version, these Transformer models are difficult to scale on datasets like \textsc{PubMed} and \textsc{Deezer} with more than 10K nodes. As compared in Table~\ref{tab:computation}, the memory cost of SGFormer-GCN is only 5.7\% and 25\% of that cost by Graphormer$ _\textsc{UltraSmall} $ and NodeFormer, respectively.

\textbf{Scalability Test.} We further test the model's scalability w.r.t. the numbers of nodes within one computation batch. We adopt the \textsc{Amazon2M} dataset and randomly sample a subset of nodes with the node number ranging from 10K to 100K. For fair comparison, we use the same hidden size 256 for all the models. In Fig.~\ref{fig-scalability}, we can see that the time and memory costs of SGFormer both scale linearly w.r.t. graph sizes. When the node number goes up to 40K, the model (SGFormer w/ Softmax) that replaces our attention with the Softmax attention suffers out-of-memory and in contrast, SGFormer costs only 1.5GB memory. 

\subsection{Results with Limited Labeled Data}\label{sec-exp-ratio}

As extension of the benchmark results, we next test the model with different amount of labeled data to investigate into the generalization capability of the model with limited training data. We consider three datasets \textsc{cora}, \textsc{pubmed} and \textsc{chameleon} and randomly split the nodes in each dataset into training/validation/testing with different ratios. In specific, we fix the validation and testing ratios both as 25\% and change the training ratio from 50\% to 1\%. Table~\ref{tab:ratio-res-cora}, \ref{tab:ratio-res-pubmed} and \ref{tab:ratio-res-chameleon} present the experimental results on three datasets, respectively, where we compare SGFormer with the GNN and Transformer competitors. We found that as the training ratio decreases, the performance of all the models exhibits degradation to a certain degree. In contrast, SGFormer maintains the superiority and achieves the best testing scores across most of the cases. This validates that our model possesses advanced ability to learn with limited training data, attributing to the simple and light-weighted architecture.

\begin{figure*}[t]
    \centering
    \includegraphics[width=\textwidth]{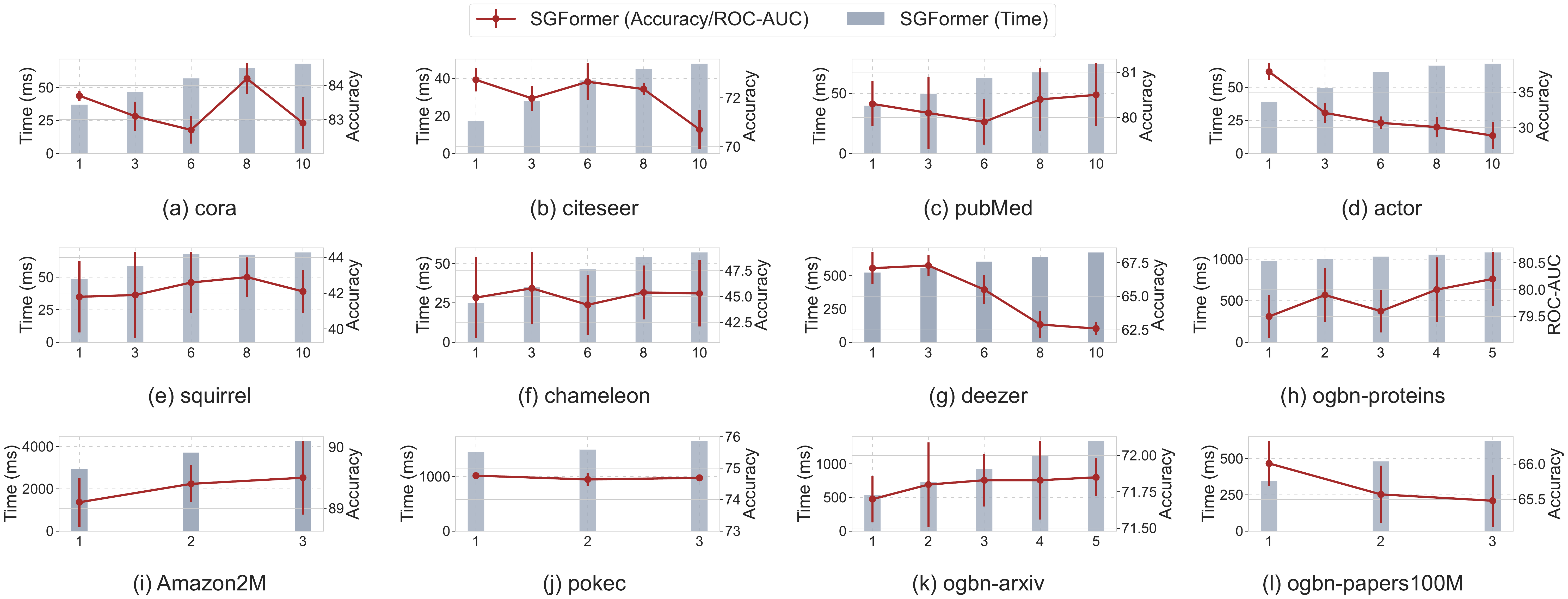}
    % \vspace{-5pt}
    \caption{Comparison of single-layer v.s. multi-layer models on 12 experimental datasets. In each dataset, we plot the training time cost per epoch and testing scores (Accuracy/ROC-AUC) of SGFormer w.r.t. the number of attention layers.}
    \label{fig:layer_ablation-sgformer}
    % \vspace{-5pt}
\end{figure*}

\begin{table}[tb!]
\centering
\caption{Ablation study on attention functions.}
\vspace{-5pt}
\label{tbl:ablation-attn}
\resizebox{\linewidth}{!}{
\begin{tabular}{l|c|c|c}
\toprule
\textbf{Dataset} & \multicolumn{1}{c|}{\textsc{cora}} & \multicolumn{1}{c|}{\textsc{actor}} & \multicolumn{1}{c}{\textsc{ogbn-proteins}} \\
 \midrule
 SGFormer (default) & 89.22 ± 0.18 & 37.36 ± 1.97 & 79.18 ± 0.44 \\
 w/ GAT Attention & 83.21 ± 0.96 & 35.81 ± 1.31 & 73.20 ± 0.68 \\
    w/ Softmax Attention & 76.47 ± 0.70 & 27.59 ± 1.21 & 68.34 ± 0.53 \\
    w/ NodeFormer Attention & 81.33 ± 2.14 & 35.93 ± 0.96 & 72.91 ± 0.78\\
\toprule
\end{tabular}
}
\vspace{-5pt}
\end{table}

\begin{table}[tb!]
\centering
\caption{Performance comparison with different $\alpha$.}
\vspace{-5pt}
\label{tbl:ablation-alpha}
\resizebox{\linewidth}{!}{
\begin{tabular}{l|c|c|c}
\toprule
\textbf{Dataset} & \multicolumn{1}{c|}{\textsc{cora}} & \multicolumn{1}{c|}{\textsc{actor}} & \multicolumn{1}{c}{\textsc{ogbn-proteins}} \\
 \midrule
 SGFormer ($\alpha=0.8$) & 89.22 ± 0.18 & 35.64 ± 3.07 & 79.04 ± 0.18 \\
 SGFormer ($\alpha=0.5$) & 88.74 ± 0.27 & 37.36 ± 1.97 & 79.12 ± 0.40 \\
    SGFormer ($\alpha=0.2$) & 75.42 ± 0.57 & 37.16 ± 1.58 & 79.18 ± 0.44 \\
    SGFormer ($\alpha=0$) & 75.57 ± 0.61 & 36.75 ± 1.74 & 73.06 ± 0.34 \\
\bottomrule
\end{tabular}
}
\vspace{-5pt}
\end{table}

\begin{figure*}[t]
    \centering
    \includegraphics[width=\textwidth]{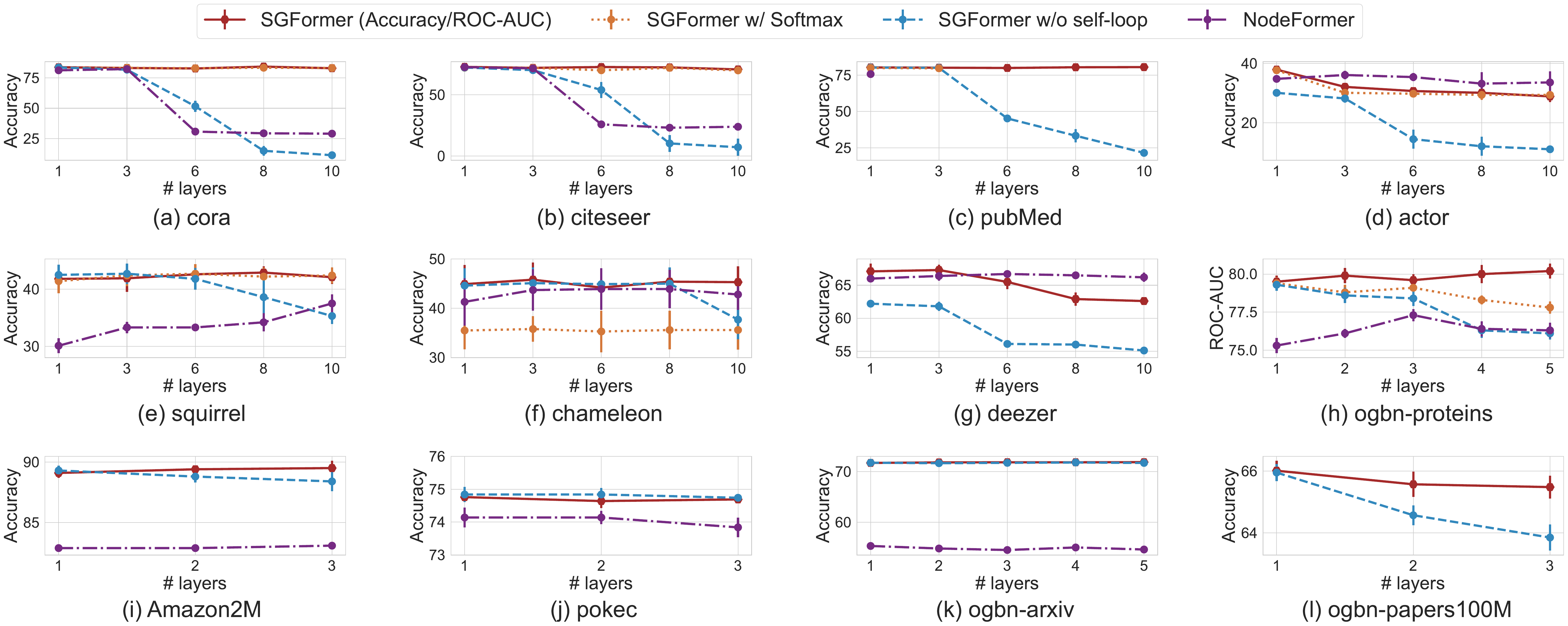}
    % \vspace{-5pt}
    \caption{Performance comparison of single-layer v.s. multi-layer models on 12 experimental datasets. In each dataset, we plot the testing scores (Accuracy/ROC-AUC) of SGFormer, SGFormer w/o self-loop, SGFormer w/ Softmax and NodeFormer w.r.t. the number of attention layers.}
    \label{fig:layer_ablation-other}
    % \vspace{-5pt}
\end{figure*}

\subsection{Ablation Study and Hyper-parameter Analysis}\label{sec-exp-ablation}

Apart from the comparative experiments, we provide more discussions to justify the effectiveness of our model, including ablation study w.r.t. different attention functions and analysis on the impact of the hyper-parameter $\alpha$.

\textbf{Impact of Attention Function.} SGFormer adopts a new attention function for computing the all-pair interactions to compute the updated node embeddings. To verify its effectiveness, we replace it by the GAT~\cite{GAT}'s attention, Softmax attention used by \cite{transformer} and NodeFormer~\cite{nodeformer}'s attention. Table~\ref{tbl:ablation-attn} presents the results using different attention functions on \textsc{cora}, \textsc{actor} and \textsc{proteins}. We found that the default attention function yields the best results across three datasets, which suggests that our design is superior for learning node representations. In particular, the vanilla Softmax attention leads to unsatisfactory performance compared to other choices. This is possibly due to the over-normalizing issue as identified by \cite{nodeformer}, i.e., the normalized attention scores after the Softmax operator tend to attend on very few nodes and cause gradient vanishing for other nodes during optimization. In contrast, the linear attention function used by SGFormer directly computes the all-pair similarity via dot-product and can overcome the over-normalizing issue. 

\textbf{Impact of $\alpha$.} The hyper-parameter $\alpha$ controls the weights on the GN module and the all-pair attention for computing the aggregated embedding. To be specific, larger $\alpha$ means that more importance is attached to the GN module and the input graph for computing the representations. In Table~\ref{tbl:ablation-alpha} we report the results with different $\alpha$'s on three datasets, which shows that the optimal setting of $\alpha$ varies case by case. In particular, $\alpha=0.8$ yields the best result on \textsc{cora}, presumably because this dataset has high homophily ratio and the input graph is useful for the predictive task. Differently, for \textsc{actor}, which is a heterophilic graph, we found using $\alpha=0.5$ achieves the best performance. This is due to that in such a case, the observed structures can be noisy and the all-pair attention becomes important for learning useful unobserved interactions. The similar case applies to \textsc{proteins}, which is also a heterophilic graph, and we found setting a relatively small $\alpha$ leads to the highest testing score. In fact, the informativeness of input graphs can vary case by case in practice, depending to the property of the dataset and the downstream task. SGFormer allows enough flexibility to control the importance attached to the input structures by adjusting the weight $\alpha$ in our model.

\subsection{Further Discussions}\label{sec-exp-dis}

We proceed to analyze the impact of the number of model layers on the performance and further compare multi-layer and single-layer models on different datasets. Fig.~\ref{fig:layer_ablation-sgformer} and \ref{fig:layer_ablation-other} present the experimental results across 12 datasets.

\textbf{Single-layer v.s. multi-layer attentions of SGFormer.} In Fig.~\ref{fig:layer_ablation-sgformer}, we plot the training time per epoch and the testing performance of SGFormer when the layer number of global attention increases from one to more. We found that using more layers does not contribute to considerable performance boost and instead leads to performance drop in some cases (e.g., the heterophilic graphs \textsc{actor} and \textsc{deezer}). Even worse, multi-layer attention requires more training time costs. Notably, using one-layer attention of SGFormer can consistently yield highly competitive performance as the multi-layer attention. These results verify the theoretical results in Sec.~\ref{sec-theory} and the effectiveness of our single-layer attention that has desired expressivity and superior efficiency.

\textbf{Single-layer v.s. multi-layer attentions of other Transformers.} In Fig.~\ref{fig:layer_ablation-other}, we present the testing performance of SGFormer, NodeFormer, SGFormer w/o self-loop (removing the self-loop propagation in Eqn.~\ref{eqn-new-attention}) and SGFormer w/ Softmax (replacing our attention by the Softmax attention), w.r.t. different numbers of attention layers in respective models. We found that using one-layer attention for these models can yield decent results in quite a few cases, which suggests that for other implementations of global attention, using a single-layer model also has potential for competitive performance. %In some cases, for instance, ANS-GT suffers performance drop on \textsc{cora} using one layer. The reason is that ANS-GT does not utilize global all-pair attention, and decreasing the layer number would degrade the model's expressivity. 
In some cases, for instance, NodeFormer produces unsatisfactory results on \textsc{ogbn-proteins} with one layer. This is possibly because NodeFormer couples the global attention and local propagation in each layer, and using the single-layer model could sacrifice the efficacy of the latter. On top of all of these, we can see that it can be a promising future direction for exploring  effective shallow attention models that can work consistently and stably well. We also found when using deeper model depth, SGFormer w/o self-loop exhibits clear performance degradation and much worse than the results of SGFormer, which suggests that the self-loop propagation in Eqn.~\ref{eqn-new-attention} can help to maintain the competitiveness of SGFormer with multi-layer attentions.

% \emph{.} In Fig.~\ref{fig:layer} we increase the layer number $K$ from 1 to 20, and the performance of \ours keeps stable, which shows that \ours is immune to the potential over-smoothing which is commonly encountered by deep GNNs. Moreover, we found that using extremely shallow layers (e.g., 1 or 2) is sufficient for satisfactory performance, which can be an advantage of \ours.

% \textbf{The Impact of $\alpha$ and $\gamma$.} In Fig.~\ref{fig:residual} of Appendix~\ref{appx-result}, we plot the accuracy w.r.t. the combination weights $\alpha$ and $\gamma$. The results show that the residual weight $\alpha$ has negligible effect on the accuracy. In particular, small $\gamma$ (less concentration on input graphs) leads to some performance degradation, which implies that input structures are informative and our graph-enhanced propagation can indeed leverage such information.

% \textbf{Ablation Studies on Attention Functions.} We replace the attention definition in \eqref{eqn-attn-ours} with the dot-then-exponential attention in \cite{transformer} and the concatenate-then-transform attention in \cite{GAT}. Both of them have $O(N^2)$ complexity. The results in Fig.~\ref{fig:attn} located in Appendix~\ref{appx-result} show that our attention mechanism can produce competitive accuracy and notably, significantly reduce the time cost, thanks to its linear complexity.

\section{Conclusions}

This paper aims at unlocking the potential of simple Transformer-style architectures for learning large-graph representations where the scalability challenge plays a bottleneck. By analysis on the learning behaviors of Transformers on graphs, we reveal a potential way to build powerful Transformers via simplifying the architecture to a single-layer model. Based on this, we present Simplified Graph Transformers (SGFormer) that possesses desired expressiveness for capturing all-pair interactions with the minimal cost of one-layer attention. The simple and lightweight architecture enables to scale smoothly to a large graph with 0.1B nodes and yields significant acceleration over peer Transformers on medium-sized graphs. On top of our technical contributions, we believe the results could shed lights on a new promising direction for building powerful and scalable Transformers on large graphs, which is largely under-explored. For future work, we plan extend SGFormer to the setting of solving combinatorial problems~\cite{Yan2020LearningFG,WangPAMI23,JiangPAMI21} whereby the GNN has been a popular backbone.

\ifCLASSOPTIONcaptionsoff
  \newpage
\fi

\bibliographystyle{IEEEtran}
\bibliography{IEEEabrv,Bibliography}

\vfill

% Can be used to pull up biographies so that the bottom of the last one
% is flush with the other column.
%\enlargethispage{-5in}

% that's all folks
\end{document}